\documentclass{article} 
\usepackage{cite}
\usepackage{amsmath,amssymb,amsfonts}
\usepackage[linesnumbered,lined,boxed,commentsnumbered,ruled]{algorithm2e}
\allowdisplaybreaks
\usepackage{url}
\usepackage{tikz}
\usepackage{soul}
\usepackage{makecell}
\usepackage{booktabs}
\usepackage{multirow}
\usepackage{colortbl}
\usepackage{pgfplots}
\usepackage{caption}
\usepackage{subcaption}
\usepackage{amsthm}
\allowdisplaybreaks
\usepackage{bm}
\pgfplotsset{compat=newest}
\newcommand{\tabl}[2]{\begin{tabular}{#1} #2 \end{tabular}}

\usepackage[disable]{todonotes}

\newtheorem{theorem}{Theorem}

\newtheorem{proposition}{Proposition}
\newtheorem{lemma}{Lemma}
\usepackage{chngcntr}
\newtheorem{assumption}{}

\allowdisplaybreaks

\newcommand{\R}{\mathbb{R}}
\newcommand{\E}{\mathbb{E}}
\newcommand{\tr}{^\mathsf{\scriptscriptstyle T}}

\usepackage{todonotes}
\usepackage{xspace}
\newcommand{\todoi}[2][]{\xspace\todo[color=red!20!white,size=\scriptsize,#1,inline]{#2}}
\newcommand\numberthis{\addtocounter{equation}{1}\tag{\theequation}}

\DeclareMathOperator*{\argmax}{argmax}

\usepackage{cleveref}
\crefname{subsection}{section}{sections}
\crefname{lemma}{lemma}{lemmas}
\crefname{property}{property}{properties}
\crefname{table}{table}{tables}
\crefname{assumption}{assumption}{assumptions}

\newcommand{\Hess}{\mathcal{H}}

\newtheorem{definition}{Definition}
\usepackage{pifont}
\newcommand{\cmark}{\ding{51}}%
\newcommand{\xmark}{\ding{55}}

\title{Policy Newton methods for Distortion Riskmetrics}

\author{
Soumen Pachal$^*$\\
{\normalsize Indian Institute of Technology Madras} \\
{\normalsize \texttt{cs22d009@smail.iitm.ac.in}}
\and 
 Mizhaan Prajit Maniyar$^*$ \\
{\normalsize Google Deepmind} \\
{\normalsize \texttt{mizhaan@google.com}}
\and
Prashanth L. A.\\
{\normalsize Indian Institute of Technology Madras}\\
{\normalsize \texttt{prashla@cse.iitm.ac.in}}
}
\date{}

\begin{document}
\maketitle
\def\thefootnote{*}\footnotetext{Equal contribution}

\begin{abstract} 
We consider the problem of risk-sensitive control in a reinforcement learning (RL) framework. In particular, 
we aim to find a risk-optimal policy by maximizing the distortion riskmetric (DRM) of the discounted reward in a finite horizon Markov decision process (MDP). 
DRMs are a rich class of risk measures that include several well-known risk measures as special cases.
We derive a policy Hessian theorem for the DRM objective using the likelihood ratio method.
Using this result, we propose a natural DRM Hessian estimator from sample trajectories of the underlying MDP.
Next, we present a cubic-regularized policy Newton algorithm for solving this problem in an on-policy RL setting using estimates of the DRM gradient and Hessian. 
Our proposed algorithm is shown to converge to an $\epsilon$-second-order stationary point ($\epsilon$-SOSP) of the DRM objective, and this guarantee ensures the escaping of saddle points. The sample complexity of our algorithms to find an $\epsilon$-SOSP is $\mathcal{O}(\epsilon^{-3.5})$. Our experiments validate the theoretical findings. To the best of our knowledge, our is the first work to present convergence to an $\epsilon$-SOSP of a risk-sensitive objective, while existing works in the literature have either shown convergence to a first-order stationary point of a risk-sensitive objective, or a SOSP of a risk-neutral one.
\end{abstract}

\section{Introduction} 
\label{sec:introduction}
Reinforcement learning (RL) has achieved tremendous success in several applications, e.g., finance, transportation, insurance, and supply chain management to name a few. In a traditional RL setting, the aim is to learn an optimal policy that maximizes the expected sum of discounted rewards. However, the expected value may not be a good objective in practical applications. To illustrate, consider a portfolio optimization application, where the goal is to find an optimal way to rebalance the portfolio that is spread over assets with varying risks \cite{cover1991universal}. In this application, an appealing strategy is to invest in assets with high risk but high return. The portfolio could involve safe stocks with low growth (and low risk), but \cite{cover1991universal} showed that volatile stocks lead to great gains. A similar viewpoint is echoed in the well-known Markowitz theory of portfolio optimization \cite{markowitz1952portfolio}. 

In this paper, we focus on distortion riskmetrics (DRMs) \cite{Wang2020_ES} --- a general class that covers several well-known risk measures as special cases, e.g., value at risk (VaR) \cite{jorion1996risk2}, conditional value at risk (CVaR)\cite{rockafellar2000optimization}, Gini deviation \cite{Gini1912}, Gini shortfall \cite{furman2017gini}, rank-dependent expected utility \cite{quiggin2012generalized}.
DRMs distort the original distribution by using a distortion function, say $h:[0,1]\rightarrow [0,1]$, with $h(0)=0$, and then calculate a distorted expected value.
More importantly, DRMs include distortion risk measures \cite{wang1996premium,denneberg1990}, which additionally assume $h(1)=1$ and also that $h$ is monotone. DRMs also include deviation measures such as Gini deviation, Gini shortfall, Wang's right-tail, left-tail, and two-sided deviations \cite{jones2003empirical}. 
DRMs are also equivalent to
spectral risk measures \cite{acerbi2002spectral}, see \cite{gzyl06}. 
DRMs represent a rich class of risk measures, as shown recently by a characterization result involving a combination of DRM-type risk measures in \cite{williamson2024drm}. 
The reader is referred to Figure \ref{fig:w} and Table \ref{tab:distortion_function_example} for several examples of distortion functions. 

\begin{table*}
	\centering
 \caption{Comparison of sample complexities for finding either a $\epsilon$-first-order stationary point ($\epsilon$-FOSP) or $\epsilon$-second-order stationary point ($\epsilon$-SOSP), see \Cref{def:fosp}. A \cmark \; indicates that the algorithm converges to the first or second-order stationary point, whereas \xmark \; implies the algorithm doesn't converge to the corresponding stationary point. The first four rows are risk-neutral RL algorithms, while the last two correspond to risk-sensitive RL with DRM as the risk measure.}
        \begin{tabular}{|c|c|c|c|c|}
			\hline 
			\multirow{2}{*} {\textbf{Algorithm}}& \multirow{2}{*} {\textbf{Objective}} &\multirow{2}{*}{\textbf{Sample Complexity}} & \multirow{2}{*}{\textbf{$\epsilon$-FOSP}} & \multirow{2}{*}{\textbf{$\epsilon$-SOSP}} \\ 
			& &  & &  \\[0.1ex] \hline
 			REINFORCE \cite{williams1992simple}  & Expected value & $\mathcal{O} \left(\frac{1}{\epsilon^4} \right)$ &   \cmark  & \xmark \\[0.2ex] \hline
			HAPG\cite{shen2019hessian} & Expected value & $\mathcal{O} \left(\frac{1}{\epsilon^3} \right) $ & \cmark & \xmark \\[0.2ex] \hline
		\cite{yang2021sample} & Expected value & $\mathcal{O} \left(\frac{1}{\epsilon^{4.5}} \right) $ & \cmark & \cmark \\[0.2ex] \hline
				CR-PN \cite{maniyar2024crpn} & Expected value &     $\mathcal{O} \left(\frac{1}{\epsilon^{3.5}} \right) $ & \cmark & \cmark \\[0.2ex]\hline
        DRM-OnP-LR\cite{vijayan2021policy} & DRM &     $\mathcal{O} \left(\frac{1}{\epsilon^{2}} \right) $ & \cmark & \xmark \\[0.2ex]\hline
     Our work (CRPN-DRM) & DRM &     $\mathcal{O} \left(\frac{1}{\epsilon^{3.5}} \right) $ & \cmark & \cmark \\[0.1ex]\hline
        \end{tabular}
  \label{tab:sample_complexity_comparison}
\end{table*}



We consider a risk-sensitive RL problem with DRM as the objective in a finite horizon Markov decision process (MDP). The aim is to learn an optimal policy by maximizing the DRM of the cumulative discounted reward. 
We adopt the policy gradient solution approach for this problem. 
Risk-sensitive RL via policy gradients has received a lot of research attention recently, see \cite{prashanth2022risk} for a recent survey. 
Previous works have explored specific DRMs such as Gini deviation \cite{Luo2023}, inter-expected shortfall range \cite{Han2022} and an abstract distortion risk measure \cite{vijayan2021policy}. However, distortion riskmetrics in their generality have not been considered earlier. The closest related previous work in \cite{vijayan2021policy} has established convergence of a policy gradient algorithm to a first-order stationary point, which include spurious saddle points that do not result in a maxima for DRM. Moreover, 
there is no prior work that shows the practical utility of policy gradient type algorithms with a DRM objective that exhibits improved behavior from a risk-sensitivity viewpoint.
Our work aims to fill the research gaps in theory as well as practice. 

\paragraph{Our contributions.} We summarize our contributions below, whereas Table \ref{tab:sample_complexity_comparison} compares sample complexities of our algorithm to closely related works in the literature on risk-neutral and DRM-sensitive RL.

First, we derive a policy Hessian theorem for the DRM of cumulative discounted reward in a finite horizon MDP. 
Using this result, we employ the likelihood ratio method to arrive an estimator of the Hessian of DRM using sample trajectories. We establish MSE bounds of $\mathcal{O}(\frac{1}{b^{3/2}})$ for our DRM Hessian estimator, where $b$ number of trajectories. 
Second, we propose a cubic-regularized policy Newton algorithm in the on-policy RL setting for maximizing the DRM of cumulative discounted reward in a finite horizon MDP. 
Third, we provide a non-asymptotic bound in expectation for convergence to an $\epsilon$-second-order stationary point ($\epsilon$-SOSP) of our algorithm. The sample complexity of our proposed algorithm is $\mathcal{O}(\epsilon^{-3.5})$. More importantly, our algorithm escapes saddle points. To the best of our knowledge, we are the first to present convergence of a  policy gradient-type algorithm to an $\epsilon$-SOSP convergence of a risk-sensitive objective.
Finally, we conduct simulation experiments on the three environments, namely cliff walk, cart pole and humanoid. In each case, we find that our DRM-sensitive policy Newton algorithm finds a risk-seeking policy with a higher expected return than risk-neutral variant.

\paragraph{Related work.} 
Various riskmetrics have been proposed in the literature. A popular class of riskmetrics is coherent risk measures \cite{artzner1999coherent}. A coherent risk measure is sub-additive, homogeneous, translation invariant, and monotonic. Value-at-risk (VaR) is a popular risk measure that is widely used in financial applications. However, VaR is not coherent. A closely related risk measure is Conditional value at risk (CVaR) \cite{rockafellar2000optimization}, known variously as tail VaR, and expected shortfall. Unlike VaR, CVaR is a coherent risk measure. Distortion risk measures generalize VaR and CVaR, and under a concave distortion function, these measures are coherent. Distortion riskmetrics (DRMs), as mentioned before, generalize distortion risk measures further by including non-monotone distortion functions. Examples of non-monotone DRMs include Gini deviation, Gini shortfall, mean-median deviation, inter-quantile range, inter-ES range etc.
 A popular non-coherent risk measure called cumulative prospect theory (CPT) was proposed in \cite{tversky1992advances,prashanth2016cumulative}. 
CPT has been shown to model human preferences well through an inverted S-shaped distortion function. A risk measure that is closely related to CPT is rank-dependent expected utility (RDEU), see \cite{quiggin2012generalized}.


\todoi{Clean from here}

In the context of planning, where the underlying MDP model is known, dynamic programming (DP) approaches have been employed to find a time-consistent optimal policy of a risk-sensitive MDP.  Many such approaches have been proposed in the literature for both finite and infinite-horizon MDPs. For instance, average-value-at-risk has been studied in \cite{bauerle2011markov} to find a time-consistent optimal policy in a discounted cost MDP. In \cite{li2022quantile}, the authors focus on optimizing quantiles instead of the usual expected return, while entropic risk measure (ERM) and entropic value-at-risk (EVaR) have been studied in \cite{hau2023entropic} for discounted MDPs. Optimized certainty equivalent (OCE) risk measure has been discussed in \cite{bauerle2024markov} for non-stationary policies in infinite-horizon MDP. 
The aforementioned works operate in a setting where the MDP model is known and 
their proposed methods do not lend easily to an RL setting. For instance, in \cite{li2022quantile}, the authors use DP to solve an auxilary MDP with infinite state space. In \cite{hau2023entropic}, the authors derive a Bellman type equation for entropic risk, which enables the application of a DP algorithm. However, this equation is not amenable for RL since it involves a distribution in the RHS of the Bellman equation, which is not directly available via sampling. 
Several risk measures, including DRMs, do not have a dynamic programming scheme that is amenable to the learning scenario that is typical in a RL setting.

Risk-sensitive RL using the policy gradient approach has received a lot of research attention recently, and a variety of risk measures including CVaR, mean-variance tradeoff, cumulative prospect theory, percentile criteria, have been explored, see \cite{prashanth2022risk} for a survey. 
In \cite{anantharam2017variational}, the authors consider an exponential sum of rewards in this example, and aim to find a risk-sensitive re-balancing strategy. By approximating the exponential sum of rewards as a sum of mean and a constant multiple of variance, we can intuitively reason that the aforementioned risk-sensitive strategy prefers volatile stocks with high returns. A similar interpretation in a closely related RL setting can be seen in \cite{chow2020variational}. 

In \cite{maniyar2024crpn}, authors proposed a risk-neutral cubic-regularized policy Newton algorithm in a finite horizon MDP and show the convergence to a $\epsilon$-SOSP. In comparison, we consider a risk-sensitive objective based on DRMs. Unlike the risk-neutral case, our algorithm requires estimation of the DRM Hessian, which in turn requires empirical distribution functions (or sample means are not enough). Our DRM policy Hessian theorem and the Hessian estimate analysis (for an MSE bound) involve significant deviations from the risk-neutral case considered in \cite{maniyar2024crpn}.

In \cite{vijayan2021policy}, the authors propose a policy gradient algorithm for distortion risk measures and establish the convergence to an $\epsilon$-FOSP. However, such a convergence guarantee is insufficient, as FOSPs often include saddle points. In contrast, our policy Newton algorithm avoids saddle points.
More importantly, we consider a larger class of distortion riskmetrics that subsume distortion risk measures, and unlike \cite{vijayan2021policy}, we also demonstrate the practical utility of DRM-sensitive algorithm in three well-known RL benchmarks. 

In \cite{tamar2015coherent}, the authors propose a policy gradient approach for a coherent risk measure and this subsumes DRMs under some conditions on the distortion function. However, their approach requires the solution of an optimization problem for risk estimation and gradient computation, and they establish asymptotic convergence to a FOSP. In contrast, we have gradient estimates that are easy to implement, and we establish a non-asymptotic bound that shows convergence to an SOSP.


\section{Problem formulation} 
\label{sec:problem_formulation}
Let $X$ be a random variable and let $F$ denote the cumulative distribution function (CDF) of $X$. The distortion riskmetric (DRM) of $X$ is defined by
\begin{align}
    \rho_h (X) \!=\! 
    &\int\limits_{-\infty}^0 [h(1 - F(x))-h(1)] dx + \int\limits_0^{\infty} h(1 - F(x)) dx,\label{eq:drm-def}
\end{align}
where the distortion function $h$
is a function of bounded variation with $h(0) = 0$. 
DRM $\rho_h$ is defined only for functions $h$ that ensure both integrals in \eqref{eq:drm-def} are finite. For sub-Gaussian distributions, it is easy to see that $\rho_h(\cdot)$ is finite.
Further, if $h(t)=t$, then DRM coincides with the expected value of $X$.
Figure \ref{fig:w} presents a few popular DRMs. 

\begin{figure}[ht]
\centering
\tabl{c}{
  \scalebox{0.65}{\begin{tikzpicture}
  \begin{axis}[width=11cm,height=6.5cm,legend pos=south east,
           grid = major,
           grid style={dashed, gray!30},
           xmin=0,     
           xmax=1,    
           ymin=0,     
           ymax=1,   
           axis background/.style={fill=white},
           ylabel={\large Distortion $\bm{h(p)}$},
           xlabel={\large Probability $\bm{p}$}
           ]
          \addplot[domain=0:1, yellow, thick]{x}; 
             \addlegendentry{Identity}
          \addplot[domain=0:1, blue, thick,smooth,samples=1500] 
             {1-pow(1-x,2)}; 
             \addlegendentry{ Dual power}
          \addplot[domain=0:1, red, thick,smooth,samples=1500] 
             {exp(-sqrt(-ln(x)))}; 
             \addlegendentry{ RDEU}
          \addplot[domain=0:1, green, thick,smooth,samples=1500] 
             {x - x*x}; 
             \addlegendentry{Gini-Deviation}
          \addplot[domain=0:1, orange, thick,smooth,samples=1500] 
             {pow(x,1/2)}; 
             \addlegendentry{ PHT}
  \end{axis}
  \end{tikzpicture}}\\[1ex]
}
\caption{Examples of distortion functions. The choices
$h(t) = t$, $h(t) = 1 - (1-t)^2$, $h(t) = 1 - (1-t)^2$, $h(t) = t - t^2$, $h(t) = t^{0.5} $ correspond to identity, dual power DRM, RDEU, Gini deviation, and proportional Hazard transform (PHT) respectively. 
}
\label{fig:w}
\end{figure}
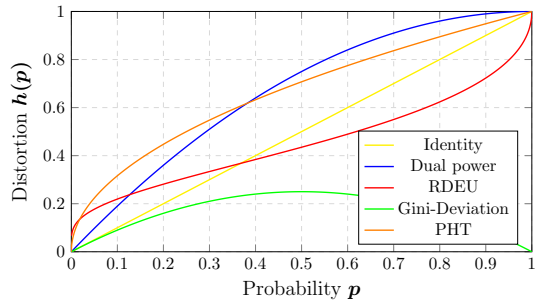
We integrate DRM into risk-sensitive RL problems to learn an optimal policy in the finite-horizon Markov decision process (MDP). We consider a finite-horizon MDP with finite state space $\mathbb{S}$, finite action space $\mathbb{A}$ . A single stage reward $r$ is defined as $r:\mathbb{S}\times\mathbb{A}\times\mathbb{S} \rightarrow [-r_{\max}, r_{\max}]$, where $r_{\max} \geq 0$. Let $p$ be the probability transition function. Let the horizon or episode length be $T$, i.e., all trajectories end after $T$ steps.  
We parameterize the stochastic policies $\{\pi_{\theta}:\mathbb{S}\times \mathbb{A} \times \mathbb{S} \rightarrow [0,1]\}$ using $\theta \in \mathbb{R}^d$. 


The cumulative discounted reward $R^{\theta}$ is given by
\begin{equation}
    R^{\theta} = \sum_{t = 0}^{T-1} \gamma^t r(S_t, A_t, S_{t+1}), 
\end{equation}
where $A_t \sim \pi_{\theta}(.,S_t)$ is the action chosen at time $t$, $S_{t+1} \sim p(.,S_t, A_t)$ is the next sate and $\gamma \in (0,1)$ is the discounting factor. 

The optimization objective in a DRM-sensitive MDP for a given policy parameter $\theta$ is the DRM of $R^\theta$. For notational simplicity, we shall use $\rho_h (\theta) $ to denote $\rho_h (R^\theta)$. 
Our aim is to find an optimal $\theta^*$ such that
\begin{equation}
\label{eq:optimal_theta}
    \theta^* \in \argmax_{\theta \in \mathbb{R}^d}  \rho_h (\theta). 
\end{equation}
A point $\theta^*$ is called a first-order stationary point (FOSP) if  $\nabla\rho_h(\theta^*) = 0$. 
FOSPs do not coincide with local maxima as they include saddle points. To mitigate this issue, the notion of a second-order stationary point (SOSP) is considered in optimization literature. At an SOSP, the gradient vanishes and the Hessian is positive semi-definite.

Finding an FOSP/SOSP directly in an RL setting is not feasible as the model information is not available. 
Moreover, running a policy gradient-type algorithm for a finite number of steps would ensure convergence to an approximate FOSP or SOSP.  We make the latter notions precise in the definition below.
\begin{definition} [\textbf{$\epsilon$-first and second-order stationary point}]
\label{def:fosp}
    Let $\epsilon >  0$. An output $\Bar{\theta}$ of a stochastic iterative algorithm to solve \eqref{eq:optimal_theta} is said to be $\epsilon$-first-order stationary point if 
    $\E\left[\left\|\nabla\rho_h(\Bar{\theta})\right\|\right] \leq \epsilon$ and an $\epsilon$-second-order stationary point if 
    $\max \left\{\sqrt{\E[\|\nabla \rho_h (\Bar{\theta})}\|], \frac{1}{\sqrt{\rho}} \E[\lambda_{\max}(\nabla^2\rho_h(\Bar{\theta}))]\right\} \leq \sqrt{\epsilon},$
    for some $\rho > 0$. 
\end{definition}
We are interested in finding an $\epsilon$-SOSP of $\rho_h(\cdot)$. Convergence to $\epsilon$-SOSP would imply escaping saddle points if the Hessian is not degenerate, see \cite{anandkumar2016efficient}. 


\section{DRM Policy Hessian theorem} 
\label{sec:drm_policy_Hessian_theorem}

For deriving gradient and Hessian expressions for the DRM objective, we make the following assumptions:  

\begin{assumption}
\label{ass:grad_bound}
    There exists $M_d > 0$ such that $\|\nabla \log \pi_{\theta}(a|s)\| \leq M_d$ for all $\theta \in \mathbb{R}^d, a \in \mathcal{A}, s \in S$. 
\end{assumption} 
\begin{assumption}
\label{ass:Hessian_bound}
    There exists $M_h > 0$ such that $\|\nabla^2 \log \pi_{\theta}(a|s)\| \leq M_h$ for all $\theta \in \mathbb{R}^d, a \in \mathcal{A}, s \in S.$ 
\end{assumption}
\begin{assumption}
\label{ass:g_bound}
    There exists $M_{h'}, M_{h^{''}}, M_{h^{'''}}> 0$ such that $|h'(t)| \leq M_{h'}, |h''(t)|\leq M_{h^{''}},$ and $|h'''(t)|\leq M_{h^{'''}}.$  
\end{assumption}
\begin{assumption}
\label{ass:lip_hessian}
    For any pair of parameters $(\theta_1, \theta_2)$ and any state-action $(s,a)$ there exists a $L_2$ such that 
    \begin{equation}
        \left \|\nabla^2 \log \pi_{\theta_1}(a|s) - \nabla^2 \log \pi_{\theta_2}(a|s) \right \| \leq L_2 \left \|\theta_1 - \theta_2 \right \|.  \nonumber
    \end{equation}
\end{assumption} 
\Cref{ass:grad_bound} and \ref{ass:Hessian_bound} are frequently made for analysis in policy gradient and actor-critic algorithms, as shown in \cite{shen2019hessian}. \Cref{ass:g_bound} is required for the boundedness of the DRM policy gradient and Hessian. \Cref{ass:lip_hessian} is required for the analysis of second-order policy search algorithms \cite{zhang2020global}. 

Now we present the gradient and Hessian expressions of the CDF $F_{R^{\theta}}(.)$ for policy parameter $\theta$. These are derived using the following identity: $F_{R^{\theta}}(x) = \E\left[\mathbf{1}\{R^{\theta}\leq x\}\right]$, followed by differentiation w.r.t. $\theta$ on both sides. Gradient $\nabla F_{R^{\theta}}(\cdot)$ is required in arriving at $\nabla \rho_h(\theta)$, while $\nabla^2 F_{R^{\theta}}(\cdot)$ is needed for working out the expression for $\nabla^2 \rho_h(\theta)$.
 
\begin{lemma}
\label{lemma:drm_cdf_pn}
    Suppose assumptions \ref{ass:grad_bound} -\ref{ass:Hessian_bound} hold. 
    Then for all $x \in (- M_r, M_r)$ with $M_r = \frac{r_{max}}{1 - \gamma}$, we have  
    \begin{align*}
    &
        \nabla F_{R^{\theta}}(x) = \E\left[ \mathbf{1}\{R^{\theta} \leq x\}\Phi_t\right], \textrm{ and } \\
&
\nabla^2 F_{R^{\theta}}(x)  = \E\Biggl[\mathbf{1}\{R^{\theta}\leq x\}\sum_{t = 0}^{T - 1} \nabla^2 \log \pi_{\theta}(A_t|S_t)\Biggr] + \E \Biggl[ \mathbf{1} \{R^{\theta} \leq x\} \left(\Phi_t\right)\left(\Phi_t\right)\tr \Biggr], 
\end{align*}
where $\Phi_t = \sum_{t = 0}^{T - 1} \nabla \log \pi_{\theta}(A_t|S_t)$. 
\end{lemma}
 \begin{proof}
     For the first part, see \cite{vijayan2021policy}, and for the second part,  see Appendix \ref{proof: drm_cdf_pn}. 
 \end{proof}
We next state the policy gradient and Hessian theorem that provides expressions for the DRM gradient and Hessian, which in turn involves $\nabla F_{R^{\theta}}(\cdot)$ and $\nabla^2 F_{R^{\theta}}(\cdot)$. 
\begin{theorem}[\textbf{DRM policy gradient and Hessian theorem}]
\label{theorem:drm_pn}
    Suppose \Crefrange{ass:grad_bound}{ass:Hessian_bound} hold. Then the gradient and Hessian of the DRM are given by
    \begin{align*}
        &
        \nabla \rho_h (\theta) = - \int_{M_r}^{M_r} h'(1 - F_{R^{\theta}}(x)) \nabla F_{R^{\theta}}(x) dx, \textrm{ and }  \\
    & 
        \nabla^2 \rho_h (\theta) 
         =  \int_{-M_r}^{M_r} h{''}(1 - F_{R^{\theta}}(x)) \nabla F_{R^{\theta}}(x)\nabla F_{R^{\theta}}(x)\tr dx - \int_{-M_r}^{M_r} h'(1 - F_{R^{\theta}}(x))\nabla^2 F_{R^{\theta}}(x)dx.  
    \end{align*}
\end{theorem}
 \begin{proof}
 The first part is shown in \cite{vijayan2021policy}, while the proof of the second part is available in \Cref{proof:drm_pn}. 
 \end{proof} 
Next, we establish first as well as second-order smoothness of the DRM objective under the assumptions listed above. This result is essential for the analysis of the DRM policy Newton algorithm that we propose.
\begin{lemma}[\textbf{Smoothness of DRM}]\label{lemma:hessian_liptchz}
    Let the \Crefrange{ass:grad_bound}{ass:lip_hessian} hold. Then for all $\theta_1, \theta_2 \in \mathbb{R}^d$, we have
    \begin{align}
        \|\nabla \rho_h(\theta_1) - \nabla \rho_h(\theta_2)\| &\leq G_{\mathcal{H}} \|\theta_1 - \theta_2\|,\\
        \|\nabla^2 \rho_h(\theta_1) - \nabla^2 \rho_h(\theta_2)\|&\le  L_{\mathcal{H}} \|\theta_1 - \theta_2\|,
    \end{align}
where
$    G_\mathcal{H} = 2M_rT(M_hM_{h'}+TM_d^2(M_{h'} + M_{h''}))$,
$    \nu = (TM_h + T^2M^2_d)$, $ L_{\mathcal{H}} = \xi_1 + \xi_2$ with
    $ \xi_1 = 2M_r(2M_{h''}TM_d\nu +    T^3M^3_dM_{h'''})$ and $\xi_2 = 2M_r(M_{h''}TM_d\nu + M_{h'}(TL_2 + 2 T M_dM_h))$.
\end{lemma}
 \begin{proof}
 The first part is shown in \cite{vijayan2021policy}, while the proof of the second part is available in  Appendix \ref{proof: hessian_liptchz}. 
 \end{proof} 

The expressions $\nabla \rho_h(.)$ and $\nabla^2 \rho_h (.)$ presented in \Cref{theorem:drm_pn} cannot be evaluated directly in a typical RL setting due to absence of model information.  
Instead, one can form a sample-based gradient and Hessian estimates for the DRM objective. We present these estimates in the next section.

\section{DRM policy gradient and Hessian estimation}
\label{sec:DRM policy gradient and Hessian estimation}
                                                                        
      In this section, we estimate $\nabla F_{R^{\theta}}$ and $\nabla^2 F_{R^{\theta}}$ by using $m$ and $b$ trajectories respectively. Let $R_i^{\theta}$ be the cumulative reward in $i$-th trajectory. Here, $A^i_t$ and $S_t^i$ denote the action and state at time $t$ in $i$-th trajectory. 
\paragraph{DRM gradient estimation.}
Given independent and identically distributed (i.i.d.) samples $\{R_i^{\theta},i=1,\ldots,m\}$ of $R^{\theta}$,
let $G^{m}_{R^{\theta}}(x) = \frac{1}{m}\sum_{i = 1}^{m} \mathbf{1} \{R_i^{\theta} \leq x\}, \forall x\in\R,$ denote the EDF of $F_{R^{\theta}}(.)$.
Following \cite{vijayan2021policy}, we form the estimate $\widehat{\nabla} G^m_{R^{\theta}}(.)$ of $\nabla F_{R^{\theta}}(.)$ as follows:
\begin{equation}
\label{eq:edf_gradient}
    \widehat{\nabla} G^{m}_{R^{\theta}}(x) = \frac{1}{m} \sum_{i =1}^{m} \mathbf{1}\{R_i^{\theta} \leq x\}\sum_{t = 0}^{T - 1} \nabla \log \pi_{\theta}(A^i_t| S^i_t). 
\end{equation}
Using $\widehat{\nabla} G^{m}_{R^{\theta}}(\cdot)$, we form the gradient estimate $\widehat{\nabla} \rho_h(\theta)$ as follows:
\begin{equation}
    \label{eq:drm_grad_estimation}
    \widehat{\nabla} \rho_h(\theta) = -\int_{-M_r}^{M_r}h'(1 - G^m_{R^{\theta}}(x)) \widehat{\nabla}G^m_{R^{\theta}}(x) dx. 
\end{equation}
\paragraph{DRM Hessian estimation.}
Given i.i.d. samples $\{R_i^{\theta},i=1,\ldots,b\}$ of $R^{\theta}$,
we form an estimate $\widehat{\nabla}^2 G^b_{R^{\theta}}(.)$ of $\nabla^2 F_{R^{\theta}}(.)$ as follows:
\begin{align*}
    &\widehat{\nabla}^2 G^{b}_{R^{\theta}}(x) 
     = \frac{1}{b} \sum_{i =1}^{b} \mathbf{1}\{R_i^{\theta} \leq x\}\sum_{t = 0}^{T - 1} \nabla^2 \log \pi_{\theta}(A^i_t| S^i_t) \\
    & \qquad\qquad  + \frac{1}{b} \sum_{i = 1}^{b} \mathbf{1}\{R_i^{\theta} \leq x\} \left(\sum_{t = 0}^{T - 1} \nabla \log \pi_{\theta}(A^i_t| S^i_t)\right) \times \left(\sum_{t = 0}^{T - 1} \nabla \log \pi_{\theta}(A^i_t| S^i_t)\right)\tr. \numberthis \label{eq:edf_Hessian}
\end{align*}
Using $\widehat{\nabla}^2 G^{b}_{R^{\theta}}(\cdot)$, we form the Hessian estimate $\widehat{\nabla}^2 \rho_g(\theta)$ is as follows:
\begin{align*}
    &
    \widehat{\nabla}^2 \rho_h (\theta) = \int\limits_{-M_r}^{M_r} h{''}(1 - G^{m}_{R^{\theta}}(x)) \widehat{\nabla} G^{m}_{R^{\theta}}(x) \widehat{\nabla} G^{m}_{R^{\theta}}(x)\tr dx  - \int\limits_{-M_r}^{M_r} h'(1 - G^{m}_{R^{\theta}}(x)) \widehat{\nabla}^2 G^{b}_{R^{\theta}}(x) dx. \numberthis \label{eq:drm_Hessian_estimation}
\end{align*}
The DRM Hessian estimate \eqref{eq:drm_Hessian_estimation} can be computed in a computationally efficient fashion
using order statistics of the $b$ samples $ \{R_i^{\theta}\}_{i = 1}^b$, see \Cref{lemma: drm_sample_estimator} in \Cref{sec:estimation_appendix} for details. 
\section{DRM-sensitive policy Newton algorithm } 
\label{sec:CRPN_drm}
In general, the classical Newton method fails to escape saddle points for non-convex objective functions. In \cite{nesterov2006cubic}, authors proposed a cubic-regularized Newton algorithm for escaping saddle points by adding a cubic term in the auxiliary objective. We adopt this cubic regularization technique in a risk-sensitive RL framework with DRM as the objective. 
In particular, our proposed cubic-regularized policy Newton algorithm for DRM (CRPN-DRM) performs the following update:
\begin{align*}
&
    \theta_{k+1} = \argmax_{\theta \in \mathbf{R}^d} \Biggl \{\Bigl<\widehat{\nabla }\rho_h(\theta_k),\theta - \theta_k\Bigr>  + \frac{1}{2} \Bigl<\widehat{\nabla}^2 \rho_h(\theta_k)(\theta - \theta_k),\theta - \theta_k\Bigr> - \frac{\alpha}{6}\Bigl\|\theta - \theta_k\Bigr\|^3
    \Biggr \}, 
\end{align*}
where $\alpha$ is a regularization parameter, $\widehat{\nabla }\rho_h$ and $\widehat{\nabla}^2 \rho_h$ are the gradient and Hessian estimates, defined in \eqref{eq:drm_grad_estimation} and \eqref{eq:drm_Hessian_estimation}, respectively. 
We present the pseudocode for CRPN-DRM in \Cref{sec:pseudo}.  
We can re-write the gradient \eqref{eq:drm_grad_estimation} and Hessian \eqref{eq:drm_Hessian_estimation} estimates alternatively, by first defining the following quantities:
\begin{align}
    c''_i &:= 
    \begin{cases}
    (R_{(i+1)}^{\theta} - R_{(i)}^{\theta}) h{''}(1 - \frac{i}{\tau}), & \text{$i \in \{1, \ldots, \tau-1 \}$}, \\
    (M_r - R_{(\tau)}^{\theta}) h''_{+}(0), & \text{$i = \tau$},
    \end{cases} \nonumber\\
    c'_i &:= 
    \begin{cases}
    (R_{(i+1)}^{\theta} - R_{(i)}^{\theta}) h{'}(1 - \frac{i}{\tau}), & \text{$i \in \{1, \ldots, \tau-1 \}$} \\
    (M_r - R_{(\tau)}^{\theta}) h'_{+}(0), & \text{$i = \tau$,}
    \end{cases}
    \label{eq:coeffs}
\end{align}
where $\psi'_i = \sum_{j=i}^{\tau} c'_j$, $l_{(i)}^{\theta} = \sum_{t = 0}^{T - 1} \log \pi_{\theta}(A^i_t| S^i_t)$, and $s^{\theta}_i = \sum_{j=1}^{i} l^{\theta}_{(j)}$. Then,
\begin{align*}
\widehat{\nabla} \rho_h (\theta) &= -\frac{1}{\tau}\sum_{i=1}^{\tau} \psi'_i \nabla l^{\theta}_{(i)}, \\
\widehat{\nabla}^2 \rho_h (\theta) &= \frac{1}{\tau}\!\sum_{i=1}^{\tau}  \frac{c''_i}{b} \nabla s^{\theta}_i \nabla\tr s^{\theta}_i - \psi'_i [ \nabla^2 l^{\theta}_{(i)} +  \nabla l_{(i)}^{\theta} \nabla\tr l_{(i)}^{\theta} ] .
\end{align*}
Here $\tau \in \mathbb{N}$ denotes the number of trajectories simulated and can be different for the gradient and Hessian estimates. The reader is referred to Appendix \ref{proof: drm_sample_estimator_alt} for a proof of the above equivalence for gradient and Hessian expressions. Note that the above simplified expressions make it suitable for practical implementations as it avoids redundant gradient computations. We further propose a variance-induced estimate by ignoring the cross-trajectory terms, as shown in \cite{markowitz2023}, as these terms in expectation have zero mean. The result below provides alternative expressions for DRM gradient and Hessian, while incorporating variance reduction.

\begin{lemma}[Variance-reduced DRM estimates]\label{lemma: drm_sample_estimator_alt}
    Suppose \Crefrange{ass:grad_bound}{ass:lip_hessian} hold. Then
    the DRM gradient \eqref{eq:drm_grad_estimation} and Hessian estimates \eqref{eq:drm_Hessian_estimation} form
    can be re-written as follows:
    \begin{align*}
    \widehat{\nabla} \rho_h (\theta) &= \frac{1}{m}\sum_{i=1}^{m} R^{\theta}_{(i)} h'\left(1 -\frac{i}{m}\right) \nabla l^{\theta}_{(i)}, \\
    \widehat{\nabla}^2 \rho_h (\theta) &= \frac{1}{b}\sum_{i=1}^{b} \bigg[ \psi''_i \nabla l^{\theta}_{(i)} \nabla\tr l^{\theta}_{(i)}  + R^{\theta}_{(i)} h'(1-\frac{i}{b}) \left( \nabla^2 l^{\theta}_{(i)} +  \nabla l_{(i)}^{\theta} \nabla\tr l_{(i)}^{\theta} \right) \bigg],
    \end{align*}
    where $\psi''_i = \frac{1}{b} \sum_{j=i}^{b} c''_j$, and $l_{(i)}^{\theta} = \sum_{t = 0}^{T - 1} \log \pi_{\theta} (A^i_t| S^i_t)$, with $m$ and $b$ trajectories simulated for gradient and Hessian estimation, respectively.
\end{lemma}
 \begin{proof}
 See Appendix \ref{proof: drm_sample_estimator_alt}.   
 \end{proof}

\section{Main results} 
\label{sec:main_results}

Before presenting the main $\epsilon$-SOSP convergence result, we provide the error bounds for the gradient estimate $\widehat{\nabla} \rho_h (\theta)$ and Hessian estimate $\widehat{\nabla}^2 \rho_h (\theta)$ in terms of number of trajectories $m$ and $b$, respectively. The error bound for $\widehat{\nabla} \rho_h (\theta)$ is provided in \cite{vijayan2021policy} (see \Cref{lemma:error_bound_gradient}). 

\begin{lemma} \label{lemma: error_bound_Hessian}
Suppose assumptions \ref{ass:grad_bound} - \ref{ass:lip_hessian} hold.
Let the Hessian estimate $\widehat{\nabla}^2 \rho_h (\theta)$ be computed by \eqref{eq:drm_Hessian_estimation} with $b$ number of trajectories. Suppose $m \geq b$ and $b\geq C(d)$, where $C(d) = 4(1+2\log 2d)$. Then, 
\begin{small}
    \begin{align*}   
    &
    \E \left[ \left\|\widehat{\nabla}^2 \rho_h (\theta) - \nabla^2 \rho_h (\theta)\right \|^2\right] \leq \frac{t_1 +\kappa_2}{b},\textrm{ and} 
    \\
    & \E \left[ \left\|\widehat{\nabla}^2 \rho_h (\theta) - \nabla^2 \rho_h (\theta)\right \|^3\right] \leq \frac{\sqrt{(\kappa_3 + t_2)(\kappa_2 + t_1)}}{b^{\frac{3}{2}}},\numberthis  \label{eqn:error_bound}
    \end{align*} 
\end{small}
where $\nu = (TM_h + T^2M^2_d),\; \; 
     t_1 = 32M_r^2M^2_{h'}C(d)\nu^2, \; \;t_2 = 1920M_r^2M_{h'}^4d^2\nu^4,
    \kappa_2 = 64M_r^2(3e^2 M^2_{h''}T^4M_d^4 + 2 T^4 M_d^4 M^2_{h'''} +  M^2_{h''}\nu^2), 
     \kappa_3 = 4096M_r^2[T^8M_d^8(9e^2M_{h''}^4 + 8 M_{h'''}^4) +M_{h''}^4\nu^4].$

The constants $M_h, M_d, M_{h'},M_{h''}, M_{h'''}$ are specified in \Crefrange{ass:grad_bound}{ass:lip_hessian}. 
\end{lemma}

 \begin{proof}
     See Appendix \ref{proof: error_bound_Hessian}. 
 \end{proof} 
The main result thate stablishes the convergence of our algorithm to an $\epsilon$-SOSP is given below. 
\begin{theorem}[\textbf{Convergence to $\epsilon$-SOSP}]
\label{thm: convergence_SOSP}
    Let the assumptions \ref{ass:grad_bound} - \ref{ass:lip_hessian} hold. Let $\{\theta_1, \ldots, \theta_N\}$ denote the iterates obtained by running  Algorithm \ref{alg:CRPN_drm} for $N$-iterations with the following parameters: 
    \begin{align*}
        &
        \alpha_k = 3L_{\mathcal{H}},\;\;\;\; N = \frac{12\sqrt{L_{\mathcal{H}}}\left(\rho(\theta_0) - \rho^*\right)}{\epsilon^{\frac{3}{2}}}, \\
        & m_k = \frac{25\kappa_1}{4\epsilon^2},\; \; \; \; b_k = \frac{9\sqrt[3]{2(\kappa_3 + t_2)(\kappa_2 + t_1)}}{L_{\mathcal{H}}\epsilon}, \numberthis\label{eq:hyper_parameter}
    \end{align*} 
where $\kappa_1 = 32M_r^2 T^2 M_d^2 (e^2 M_{h'}^2 + M_{h''}^2)$ and $\kappa_2, \kappa_3, t_1, t_2$ are defined in \Cref{lemma: error_bound_Hessian}.  
    Let $\Bar{\theta}$ be chosen from $\{\theta_1,\ldots, \theta_N\}$ uniformly at random. Then, for any $0<\epsilon \leq \frac{25L_{\mathcal{H}}\kappa_1}{36\sqrt[3]{2(\kappa_3 + t_2)(\kappa_2 + t_1)}}$,  we have
    \begin{equation}
        5\sqrt{\epsilon} \geq \max \left\{\sqrt{\E[\|\nabla \rho_h (\Bar{\theta})}\|], \frac{5}{6\sqrt{L_{\mathcal{H}}}} \E[\lambda_{\max}(\nabla^2\rho_h(\Bar{\theta}))]\right\}. 
    \end{equation}
\end{theorem}
 \begin{proof}
 See Section \ref{proof: convergence_SOSP}. 
 \end{proof}
As a consequence of \Cref{thm: convergence_SOSP}, the sample complexity of our CRPN-DRM algorithm to converge a $\epsilon$-SOSP is bounded by $\mathcal{O}(\epsilon^{-3.5})$. 


\section{Simulation Experiments} 
\label{sec:experiments} 
\begin{figure*}[t]
\begin{tabular}{cc}
\begin{subfigure}{0.5\textwidth} 
    \includegraphics[width=0.9\linewidth]{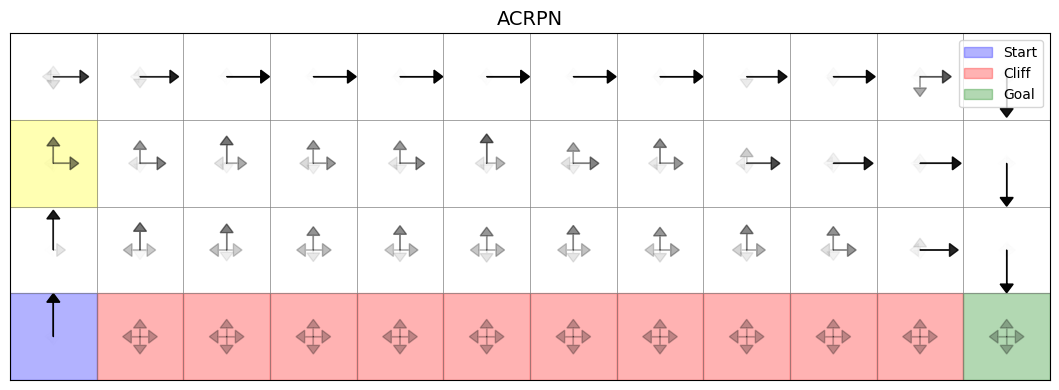}
    \label{fig:cliffwalking_policymap_drmacrpn_neutral}
\end{subfigure}&
\begin{subfigure}{0.5\textwidth} 
\includegraphics[width=0.9\linewidth]{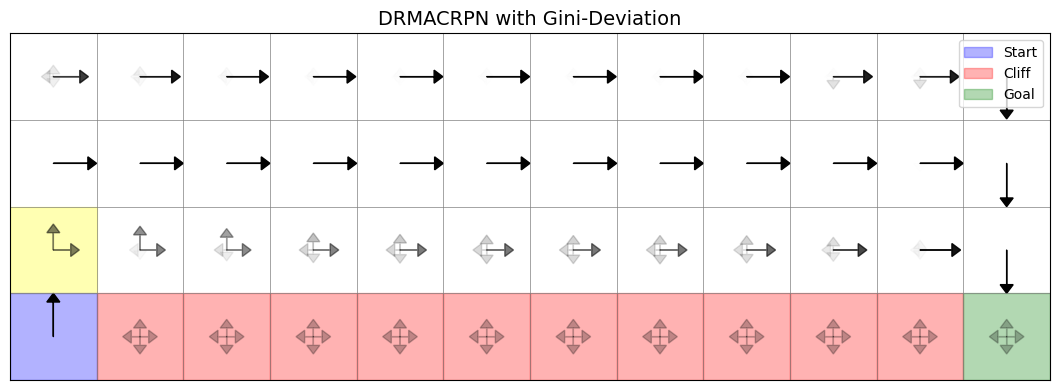}
    \label{fig:cliffwalking_policymap_drmacrpn_gini}
\end{subfigure}
\end{tabular}
\caption{A visual illustration of the policies found by ACRPN and DRMACRPN algorithms on cliff walk environment. The length of the arrows is proportional to the probability of taking that action by the policy. The bottom left and right cells denote the start and goal states, respectively.}
\label{fig:cliffwalking_policymap_results_1}
\end{figure*} 
\paragraph{Implementation}
We compare the performance of the following algorithms on three environments, namely Cart-pole, Humanoid, and cliff walking. \\
1. \textbf{ACRPN:} This is the risk-neutral cubic-regularized policy Newton algorithm from \cite{maniyar2024crpn}. We implement the approximate variant of their policy Newton algorithm, as described in Section 5 of \cite{maniyar2024crpn}. \\
2. \textbf{DRMACRPN:} This is \Cref{alg:CRPN_drm} in \cref{sec:pseudo}. We have three variants of this algorithm corresponding to three different DRMs, namely dual power, Gini deviation. See Figure \ref{fig:w} for the distortion functions for each of these DRMs. \\
3. \textbf{REINFORCE-DRM:} Uses the gradient estimate in \Cref{lemma: drm_sample_estimator_alt} and performs the traditional single step gradient ascent with the step-size $\lambda = \sqrt{\frac{2}{\alpha ||\bar{g}||}}$, 
as suggested in \cite{maniyar2024crpn}. This step-size is obtained by solving the cubic sub-problem in \eqref{eq:subproblem} while assuming $\Hess = 0$, i.e., without having access to second-order information. 

We implement \Cref{alg:CRPN_drm} on the three environments using the variance-reduced gradient and Hessian estimates as shown in \Cref{lemma: drm_sample_estimator_alt}. We consider the following two DRMs: (i) Dual-power with $h(t) = 1 - (1-t)^2$; and
(ii)  Gini-deviation with $h(t) = t - t^2$.
In our implementation of \Cref{alg:CRPN_drm}, we set the cubic-penalty parameter $\alpha=10^5$ as that gave the most stable learning behavior across DRMs. The batch sizes for gradient and Hessian estimation are set as follows: $m_k = b_k = 200$. The algorithm was run for a total of $N = 100$ iterations, using the same trajectories for the gradient as well as Hessian estimation. For ease of implementation, we assume $M_r = R^{\theta}_{(m)}$. 
The policy parameterization is tabular, linear and deep neural network on cliff walk, cart pole and humanoid environments, respectively. In particular, for the humanoid case, the policy parameterization involved a deep neural network consisting of two hidden layers of 64 neurons each, with softmax activation. All the policies use a Boltzmann distribution in the final layer for the probabilities to sum up to one. The results reported are averages over ten independent replications. 

\paragraph{Cliff walk.}
We implement a variant of this environment with a modified reward scheme which incentivizes the agent to move towards the goal state. The modified reward 
 incorporates a sense of distance from the goal state, and is given as follows:
\begin{align*}
    r_{\texttt{modified}}(s, a, s') = r_{\texttt{default}}(s, a, s') - c \cdot || s' - s_g ||_{1},
\end{align*}
where $(s, a, s')$ are the current state, current action and next state, respectively, $s_g$ is the coordinate of the goal state and $||\cdot||_1$ denotes the $\ell_1$ norm. In our experiments, we set $c = 0.5$. Note that, for the purposes of understanding the results, we only consider $r_{\texttt{default}}$ while the agent is trained on $r_{\texttt{modified}}$.

Cliff walk serves as a good environment to study the performance for the risk-neutral and DRM policies owing to the risk involved in the path close to the cliff, which can result in very high negative rewards. 

We present the mean, standard deviation, minimum and maximum of the episodic return in \Cref{tab:cliffwalking_results}. It is apparent that the risk-neutral algorithms have lower standard deviation and their \textit{ best} trajectory -- with a cumulative reward of $-16$ -- follows the path that walks along the border of the grid, farthest away from the cliff. On the other hand, risk-seeking algorithms follow the shortest ``riskier'' path -- with a cumulative reward of $-12$ -- closer to the cliff, and this path has a higher risk of falling off the cliff. 

\begin{table}[ht]
\caption{Episodic return for risk-neutral algorithms (REINFORCE, ACRPN) and their DRM counterparts with Gini deviation on cliff walking environment.}
    \centering
    \begin{tabular}{|c|c|c|c|}
        \hline
        \multirow{2}{*}{\textbf{Algorithm}} & \multirow{2}{*}{$\textbf{Mean}\pm \textbf{std}$} &  \multirow{2}{*}{\textbf{Min}} & \multirow{2}{*}{\textbf{Max}} \\ 
        & &  & \\\hline
        REINFORCE & $-16.2\pm0.5 $&  \bf{-22} & -16 \\ \hline
        ACRPN & $-16.0\pm0.9$ &  -24 & -16 \\ \hline
        REINFORCE-DRM& $-14.1\pm \bf{5.4}$ &  -123 & \bf{-12} \\ \hline
        DRMACRPN& $\bf{-13.6}\pm 4.9$ &  -123 & \bf{-12} \\ \hline
    \end{tabular}
    
    \label{tab:cliffwalking_results}
\end{table} 
To get deeper insights into the policies' actions on the grid, we present a visualization in \Cref{fig:cliffwalking_policymap_results_1}.
It is interesting to note that the less risky ACRPN policy, tends to move away from the cliff even in the first row (the row second-closest to the cliff), while the DRMACRPN policy are highly determined to move towards the goal with shorter path lengths. In the second row (the row closest to the cliff), DRMACRPN policy is more likely to move towards the goal than move away from the cliff. This highlights the urgency of the DRM-sensitive policy to get to the goal state despite the low but non-zero chance of falling off the cliff. Furthermore, noticing the cells highlighted in yellow, we see that the DRMACRPN's risk-seeking policy tends to go towards the goal earlier than the risk-neutral polices, and DRMACRPN, which is a second-order algorithm, tends to be more likely to do so than its first-order counterpart (REINFORCE-DRM), see Appendix \ref{sec:appendix-expts} for more details.

\paragraph{Humanoid.}
\begin{figure}[t]
\begin{tabular}{cc}
\scalebox{0.9}{
\begin{subfigure}{0.5\textwidth} 
    \includegraphics[width=0.9\linewidth]{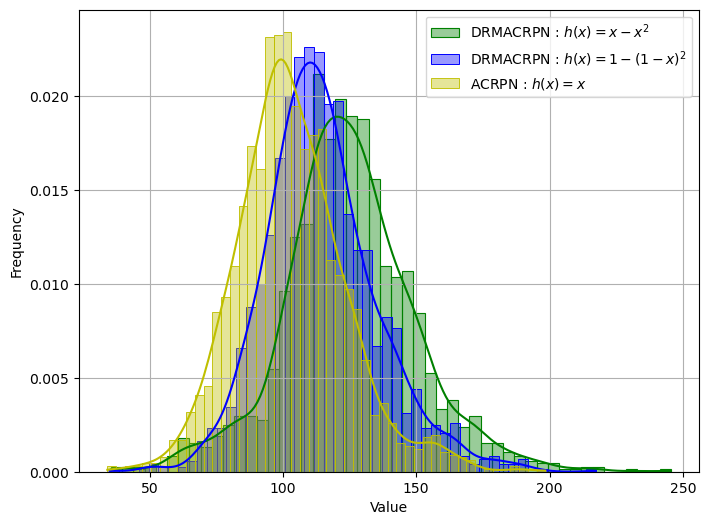}
    \caption{Episodic return distributions of the DRM policies}
    \label{fig:drm_hists}
\end{subfigure}}
&
\scalebox{0.9}{
\begin{subfigure}{0.5\textwidth} 
\includegraphics[width=0.9\linewidth]{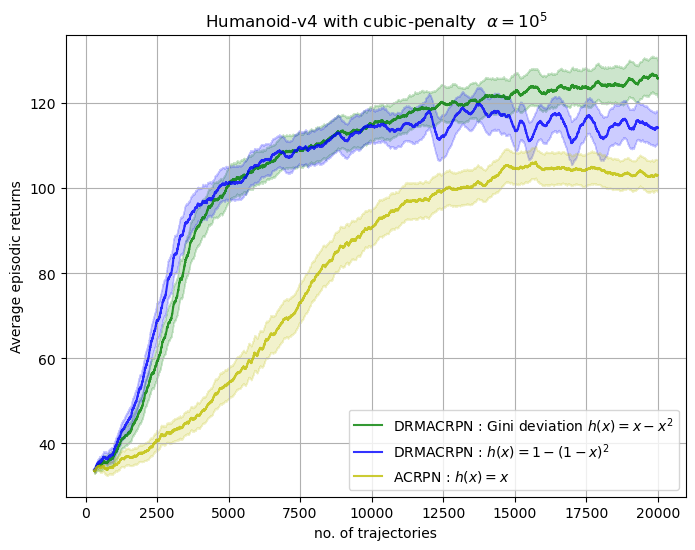}
\caption{Learning curves of the DRM policies.}
    \label{fig:drm_lc}
\end{subfigure}}
\end{tabular}
\caption{Performance comparison of risk-neutral policy Newton and \Cref{alg:CRPN_drm} with two different DRMs, namely dual-power and Gini deviation on the Humanoid-v4 environment.}
\label{fig:humanoid-results}
\end{figure} 
This environment has a three-dimensional bipedal robot designed to simulate a human.  We chose this environment as it contains both rewards and costs, and the total reward is the signed sum of these values.

\begin{figure}[ht]
\centering
\tabl{c}{
  \scalebox{0.6}{\begin{tikzpicture}
  \begin{axis}[width=11cm,height=6.5cm,legend pos=south east,
           grid = major,
           grid style={dashed, gray!30},
           xmin=0,     
           xmax=1,    
           ymin=-5,     
           ymax=5,   
           axis background/.style={fill=white},
           ylabel={\large Weight $\bm{w(x)}$},
           xlabel={\large $\bm{x}$}
           ]
           \addplot[domain=0:1, green, thick,smooth,samples=150] 
             {2*x - 1}; 
             \addlegendentry{Gini-Deviation}
          \addplot[domain=0:1, blue, thick,smooth,samples=150] 
             {2*x}; 
             \addlegendentry{Dual-power}
          \addplot[domain=0:1, yellow, thick]{1}; 
            \addlegendentry{Identity}
  \end{axis}
  \end{tikzpicture}}\\[1ex]
}
\caption{Weight coefficients for the gradient estimate of the different DRMs. Here $w(x) = h'(1-x)$ is the weight coefficient for $x\in [0,1]$. The weight associated with the return of the $i$-th trajectory is scaled by $h'(1-\frac{i}{m})$, see \Cref{lemma: drm_sample_estimator_alt}.
}
\label{fig:w3}
\end{figure}
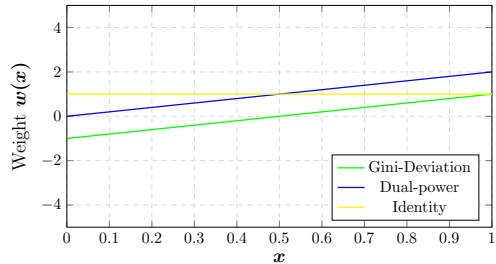
Figure \ref{fig:humanoid-results} presents the following quantities for each of the algorithms mentioned above: (i) the histogram of the policies obtained after $N$ iterations; and (ii) the learning curves, which show the change in episodic return. 
\Cref{tab:drm_results} tabulates the mean and variance of the episodic return for our DRMACRPN algorithm with different distortion functions. 
Further, \Cref{tab:drm_results_2} in Appendix \ref{sec:appendix-expts} compares the performance of our algorithm to the risk-neutral baseline. From these results, it is apparent that our algorithm performs better than the risk-neutral counterpart w.r.t. the DRM objective. To put it differently, the risk-neutral objective is not a surrogate for DRM, and we need a specialized algorithm to optimize the latter objective.
\Cref{fig:humanoid-results} presents the performance of DRM-sensitive and risk-neutral algorithms on the humanoid environment for the dual-power and Gini-deviation DRMs. We observe that our algorithm outperforms the risk-neutral policy Newton, both from the learning curves viewpoint as well as the reward distributions shown in \Cref{fig:drm_lc}. 

We believe the superior performance of dual-power and Gini deviation is due to the fact that these DRMs weigh trajectories with higher returns more as is evident from the gradient expression in \Cref{lemma: drm_sample_estimator_alt} and from the plot of weight coefficients given in \Cref{fig:w3}. In other words, each step of \Cref{alg:CRPN_drm} would be in the direction of  policies that achieve higher returns, owing to higher weights from DRM gradient/Hessian estimates, potentially leading to faster learning by \textit{focusing on the best experiences}. \Cref{fig:drm_hists} shows the distribution of the policies obtained for the different algorithms. We also note that although Gini-deviation gives us better mean, it has thicker tails as compared to the Dual-power DRM. This result is owing to the fact that Gini deviation is closely related to variance, and the algorithm learns to maximize the same.
Similar inferences can be made from the experiments on the Cart-pole environment, see \Cref{sec:appendix-expts} for the details.
\begin{table}[t]
\caption{Sample mean and variance of the episodic return for DRMACRPN with three different distortion functions on Humanoid-v4. 
	}
    \label{tab:drm_results}
    \centering
    \begin{tabular}{|c|c|c|}
    \hline
        DRM & Mean & Standard deviation \\\hline 
       Identity  & 103.0 &  20.6 \\ \hline
       Dual-power & 114.4 &  21.1 \\ \hline
       Gini deviation & \textbf{124.8} & \textbf{24.7} \\ \hline
    \end{tabular}
\end{table}

    

\section{Conclusions}
\label{sec:concl}
We proposed a policy Newton algorithm for maximizing distortion riskmetrics. We derived the policy Hessian theorem for a DRM objective, and used it to form DRM gradient/Hessian estimates from sample trajectories. We proved a non-asymptotic bound that establishes the convergence of our policy Newton algorithm to an $\epsilon$-SOSP for the DRM objective with sample complexity $\mathcal{O}(\epsilon^{-3.5})$. 
To the best of our knowledge, this is the first work to present convergence to an $\epsilon$-SOSP in a risk-sensitive RL framework. 





\bibliographystyle{plain}
\bibliography{refs}
\newpage

\appendix
\setcounter{secnumdepth}{2}
\section*{Organization of the supplementary material}
 In \Cref{sec:pseudo}, we present the pseudocode for the CRPN-DRM Algorithm discussed in \Cref{sec:CRPN_drm}. In \Cref{sec:drm_examples}, we provide a few examples of distortion riskmetrics and verify assumption \ref{ass:g_bound} for several  DRMs. In \Cref{prrof:proofs_gradient_hessin}, we provide  the proofs for the claims in \Cref{sec:drm_policy_Hessian_theorem}. In particular, this includes the proofs of \Cref{lemma:drm_cdf_pn}, \Cref{theorem:drm_pn}, and \Cref{lemma:hessian_liptchz}. 
In \Cref{sec:estimation_appendix}, we provide the proofs of the claims  in \Crefrange{sec:DRM policy gradient and Hessian estimation}{sec:CRPN_drm}. In  \Cref{sec:main_appendix}, we provide the proofs of  \Cref{thm: convergence_SOSP} and \Cref{lemma: error_bound_Hessian}.  
In \Cref{sec:eff_hessian}, we discuss an efficient computational method for the DRM gradient and Hessian estimates. Finally, in \Cref{sec:appendix-expts}, we present the additional experiments on the humanoid, cart-pole, and cliff walking environments.   

\section{CRPN-DRM Algorithm Pseudocode} \label{sec:pseudo}
\begin{algorithm*}[!ht]
	\LinesNotNumbered
	\SetKwInOut{Input}{Input}\SetKwInOut{Output}{Output}
	\Input{initial point $\theta_0 \in \mathbb{R}^d$, 
		a non-negative sequence $\{\alpha_n\}$, and number of iterations $N\geq 1$,
	}
	\For{$k\leftarrow 1$ \KwTo $N$}{
		
		Generate $m_k$ and $b_k$ trajectories for gradient and Hessian estimation \tcp*{Simulation}
		
		$\widehat{\nabla}_{k} \rho_h (\theta_k) = - \int_{-M_r}^{M_r} h'(1 - G^{m_k}_{\theta_{k-1}}) \widehat{\nabla} G^{m_k}_{R^{\theta_{k-1}}} dx 
			$\tcp*{DRM gradient estimation}
		\tcc{DRM Hessian estimation}
		
		\vspace{1ex}
		
		\begin{small}
			\centerline{$
				\widehat{\nabla}_k^2 \rho_h (\theta_k) =  - \int_{-M_r}^{M_r} h'(1 - G^{m_k}_{R^{\theta_{k-1}}}) \widehat{\nabla}^2 G^b_{R^{\theta_{k-1}}} dx  + \int_{-M_r}^{M_r} h{''}(1 - G^{m_k}_{R^{\theta_{k-1}}}) \left(\widehat{\nabla} G^{m_k}_{R^{\theta_{k-1}}}\right) \left(\widehat{\nabla} G^{m_k}_{R^{\theta_{k-1}}}\right)\tr  dx, 
				$}
			\vspace{1ex}
		\end{small}
		where $G^{m}_{R^{\theta}}(x), \widehat{\nabla} G^{m}_{R^{\theta}}(x)$ and $\widehat{\nabla}^2 G^{b}_{R^{\theta}}(x)$ are defined in \Cref{sec:DRM policy gradient and Hessian estimation}\;

        \vspace{-3ex}
        
			\begin{align*}
		\hspace{-6em}\textrm{\texttt{Cubic-regularized Newton step:}~~}		    \theta_k = \argmax_{\theta \in \mathbb{R}^d} \left\{ \tilde{\rho}_h(\theta, \theta_{k-1},\widehat{\nabla}_{k} \rho_h (\theta_k), \widehat{\nabla}_k^2 \rho_h (\theta_k),\alpha_k)\right\},
    \numberthis\label{eq:subproblem}
			\end{align*}
		where
		$\tilde{\rho}_h(x,y,\mathcal{H},h,\alpha) = \left <h, x-y \right >+\frac{1}{2}\left <\mathcal{H}(x-y),x-y\right > -\frac{\alpha}{6}\|x-y\|^3$. 
	}
	\Output{Parameter $\theta_N $. } 
	\caption{Cubic-regularized policy Newton algorithm for DRM (CRPN-DRM)} 
	\label{alg:CRPN_drm}
\end{algorithm*} 
   
\section{DRM examples}
\label{sec:drm_examples}
We provide some examples of distortion riskmetrics in \Cref{tab:distortion_function_example}. These examples are  from \cite{wang2020distortion}. 

\begin{table*}[h]
	\centering
 \caption{Examples of distortion riskmetrics (DRMs).}
 \label{tab:summary}
		\begin{tabular}{|c|c|c|c|}
			\hline
			\multirow{2}{*}{\textbf{No}} & \multirow{2}{*} {\textbf{DRM name}}& \multirow{2}{*} {\textbf{DRM expression}} &\multirow{2}{*}{\textbf{Distortion function $h(t)$}}  \\ 
			& & &   \\[1.5ex] \hline
 			 1 \cmark & Mean  & $\E[X]$ & $ t$ \\[1.5ex] \hline
			2 & VaR & $F_X^{-1}(\alpha), \alpha \in (0,1)$ &  
			    $1 \text{ if } \ t > 1 - \alpha$ and 
                $0  \text{ otherwise.}$\\[1.5ex] \hline
				3 & CVaR & $\frac{1}{1 - \alpha}\int_{\alpha}^1 F_X^{-1}(x) dx, 0<\alpha < 1$ & $ \min \Bigl\{\frac{t}{1 - \alpha},1\Bigr\}$ \\[1.5ex]\hline
                4 & RVaR & $\frac{1}{\beta - \alpha}\int_{\alpha}^{\beta} F_X^{-1}(x) dx, 0<\alpha <\beta <1$ & $ \min \Bigl\{\frac{\max\{t-1+\beta,0\}}{\beta - \alpha},1\Bigr\}$   \\[1.5ex]\hline
                5 & Mean median& \multirow{2}{*}{$\min_{x \in \mathbb{R}}\E\left [|X - x|\right ]$}  & \multirow{2}{*}{$ \min\{t, 1 - t\}$} \\
                & deviation&&\\\hline
            6 \cmark & Wang transform\textsuperscript{a} &$\int_0^1 F_X^{-1}(x) \Bigl(\Phi^{-1}(x)-\lambda\Bigr) dx, \lambda > 0$  & $ \Phi(\Phi^{-1}(t) - \lambda)$ \\[1.5ex]\hline
            7 & Proportional Hazard  & \multirow{2}{*}{$\int_0^1 F_X^{-1}(x) \alpha x^{\alpha - 1} dx$}  & \multirow{2}{*}{$ t^{\alpha}, 0 < \alpha \leq 1$} \\
            &transform &&\\\hline
            8 \cmark & Exponential & $\int_0^1 F_X^{-1}(x) \alpha e^{-\alpha  x} dx $ & $ 1 - e^{-\alpha t}, \alpha > 0$ \\[1.5ex]\hline
             9 \cmark & Dual power & $\int_0^1 F_X^{-1}(x) \alpha x^{\alpha -1} dx $ & $ 1 - (1 - t)^\alpha, \alpha \geq 1$ \\[1.5ex]\hline
              10 \cmark  & Gini deviation\textsuperscript{b} & $\frac{1}{2}\E[|X^{*} - X^{**}|]$ & $ t - t^2$ \\[1.5ex]\hline
              11   & Gini Shortfall & CVaR(X) + $\lambda \E[|X^{*}_\alpha - X^{**}_\alpha|], \lambda > 0$ & $ \min \left\{\frac{t}{1-\alpha}, \frac{2\lambda t\max\{0,(1-t-\alpha)\}}{(1-\alpha)^2}\right\}$ \\[1.5ex]\hline
		\end{tabular} \\
  \label{tab:distortion_function_example}
  {\textsuperscript{a} $\Phi$ is the CDF of the standard normal distribution and $\Phi^{-1}$ is the quantile of the standard normal distribution.}\\
{\textsuperscript{b} Here $X^{*}, X^{**}$ are the i.i.d. copies of $X$.} A \cmark indicates that the distortion function satisfies \Cref{ass:g_bound}. 
\end{table*} 

The result below provides the first three derivatives for the distortion functions underlying several popular DRMs. From these expressions, it can be clearly seen  that these DRMs satisfy \Cref{ass:g_bound}. 
The choice of distortion function impacts the analysis through \Cref{ass:g_bound}, while  the other three assumptions relate to the policy parameterization, and are frequently used in policy gradient type algorithms.     
\begin{lemma}
\label{lemma:assum_ver}
For the different DRMs defined in \Cref{tab:distortion_function_example}, the derivatives of the distortion function $h$ are given below.\\
    (1) For the Gini-deviation DRM, we have $h^{'}(t) = 1 - 2t$, $h^{''}(t) = -2 $ and $h^{'''}(t) = 0 $.   \\
    (2) For the dual power DRM, we have $h^{'}(t) = \alpha (1-t)^{\alpha -1}$, $h^{''}(t) = -\alpha (\alpha -1) (1-t)^{\alpha - 2} $ and $h^{'''}(t) = \alpha (\alpha -1) (\alpha -2)(1-t)^{\alpha -3} $. \\
    (3) For the exponential DRM, we have $h^{'}(t) = \alpha e^{-\alpha t}$, $h^{''}(t) = - \alpha^2 e^{-\alpha t} $ and $h^{'''}(t) =  \alpha^3 e^{-\alpha t} $. \\
    (4) For the mean DRM, we have $h^{'}(t) = 1$, $h^{''}(t) = 0 $ and $h^{'''}(t) = 0 $. 
\end{lemma}

\section{Proofs of the claims in \Cref{sec:drm_policy_Hessian_theorem}} 
\label{prrof:proofs_gradient_hessin}

\subsection{Proof of Lemma \ref{lemma:drm_cdf_pn}}
\label{proof: drm_cdf_pn}
\begin{proof}
For the first part, see \cite{vijayan2021policy}. 
Let $\Omega$ denote the all sample episodes with cardinality $\Lambda$. We denote $S_t(\omega)$ and $A_t(\omega)$ as the state and action at time $t$ respectively for an episode $\omega$.   

    From \cite{vijayan2021policy}, we have
    \begin{align*}
        \nabla F_{R^{\theta}}(x)         &=  \nabla \E \left[\mathbf{1}\{R^{\theta}\leq x\}\right] = \sum_{\omega \in \Omega} \mathbf{1}\{R(\omega) \leq x\}\sum_{t=0}^{T - 1} \nabla \log \pi_{\theta} (A_t(\omega)|S_t(\omega))\mathbb{P}_{\theta}(\omega)
    \end{align*}
Differentiating on both sides, we have
    \begin{align*}
    &
        \nabla^2 F_{R^{\theta}} (x)  = \sum_{\omega \in \Omega} \mathbf{1} \{R(\omega) \leq x\}\sum_{t=0}^{T-1} \left(\nabla^2 \log \pi_{\theta}(A_t(\omega)|S_t(\omega))\mathbb{P}_{\theta}(\omega)
          + \nabla \log \pi_{\theta}(A_t(\omega)|S_t(\omega))\nabla \mathbb{P}_{\theta}(\omega)\right) \\
        & = \sum_{\omega \in \Omega} \mathbf{1} \{R(\omega) \leq x\}\sum_{t=0}^{T-1} \left(\nabla^2 \log \pi_{\theta}(A_t(\omega)|S_t(\omega))\mathbb{P}_{\theta}(\omega)
         + \nabla \log \pi_{\theta}(A_t(\omega)|S_t(\omega)) \frac{\nabla \mathbb{P}_{\theta}(\omega)}{\mathbb{P}_{\theta}(\omega)} \mathbb{P}_{\theta}(\omega)\right)\\ 
         & = \E\left[\mathbf{1}\{R^{\theta}\leq x\}\sum_{t = 0}^{T - 1} \nabla^2 \log \pi_{\theta}(A_t(\omega)|S_t(\omega))\right]\\
& \qquad \qquad + \E \left[ \mathbf{1} \{R^{\theta} \leq x\} \left(\sum_{t = 0}^{T - 1} \nabla \log \pi_{\theta}(A_t(\omega)|S_t(\omega))\right) \left(\sum_{t = 0}^{T - 1} \nabla \log \pi_{\theta}(A_t(\omega)|S_t(\omega))\right)\tr \right].
    \end{align*} 
    The claim follows.
\end{proof} 
\subsection{Proof of \Cref{theorem:drm_pn}}
\label{proof:drm_pn}
\begin{proof}
We know from Theorem 1 in \cite{vijayan2021policy}, 
\begin{align*}
    \nabla \rho_h (\theta) = - \int\limits_{-M_r}^{M_r} g'(1 - F_{R^{\theta}}(c)) \nabla F_{R^{\theta}}(x) dx. 
\end{align*}
Differentiating on both sides and using the dominated convergence theorem, we obtain
\begin{align*}
    & \nabla^2 \rho_h (\theta) 
         =  \int_{-M_r}^{M_r} h{''}(1 - F_{R^{\theta}}(x)) \nabla F_{R^{\theta}}(x)\nabla F_{R^{\theta}}(x)\tr dx - \int_{-M_r}^{M_r} h'(1 - F_{R^{\theta}}(x))\nabla^2 F_{R^{\theta}}(x)dx.
\end{align*}
\end{proof}
\subsection{Proof of Lemma \ref{lemma:hessian_liptchz}}
\label{proof: hessian_liptchz}
The first part is shown in \cite{vijayan2021policy}. We state a variant of a result from \cite{vijayan2021policy} that will be used in the proof of \Cref{lemma:hessian_liptchz}.
\begin{lemma}
\label{lem:cdfbounds}
 For all $x \in (-M_r, M_r)$, \[\left\|\nabla F_{R^{\theta}}(x)\right\| \leq T M_d, \; \; \; \left\|\nabla^2 F_{R^{\theta}}(x)\right \| \leq T M_h + T^2 M_d^2. \]    
\end{lemma}
\begin{proof}
    See \cite{vijayan2021policy}. 
\end{proof}

\begin{proof}(\textit{\Cref{lemma:hessian_liptchz}})

We begin with the DRM Hessian expression given as follows: 
    \begin{align*}
    & 
        \nabla^2 \rho_h (\theta) =  \int_{-M_r}^{M_r} h{''}(1 - F_{R^{\theta}}(x)) \left(\nabla F_{R^{\theta}}(x)\right)\left(\nabla F_{R^{\theta}}(x)\right)\tr dx - \int_{-M_r}^{M_r} h'(1 - F_{R^{\theta}}(x))\nabla^2 F_{R^{\theta}}(x)dx.  
    \end{align*} 
Notice that
    \begin{align*}
    &
        \|\nabla^2 \rho_h(\theta_1) - \nabla^2 \rho_h(\theta_2)\| \\
        & =    \Biggl| \Biggl|\int_{-M_r}^{M_r}h{''}(1- F_{R^{\theta_1}}(x)) \left(\nabla F_{R^{\theta_1}}(x)\right) \left(\nabla F_{R^{\theta_1}}(x)\right)\tr dx - \int_{-M_r}^{M_r}h'(1 - F_{R^{\theta_1}}(x))\nabla^2 F_{R^{\theta_1}}(x)dx \\
        & \qquad \qquad  -\int_{-M_r}^{M_r}h{''}(1- F_{R^{\theta_2}}(x)) \left(\nabla F_{R^{\theta_2}}(x)\right) \left(\nabla F_{R^{\theta_2}}(x)\right)\tr dx + \int_{-M_r}^{M_r}h'(1 - F_{R^{\theta_2}}(x))\nabla^2 F_{R^{\theta_2}}(x)dx \Biggr|\Biggr| \\
        & \overset{(a)}{\leq} \int_{-M_r}^{M_r}\Biggl|\Biggl|h{''}(1- F_{R^{\theta_1}}(x)) \left(\nabla F_{R^{\theta_1}}(x)\right) \left(\nabla F_{R^{\theta_1}}(x)\right)\tr - h'(1 - F_{R^{\theta_1}}(x))\nabla^2 F_{R^{\theta_1}}(x) \\
        & \qquad \qquad - h{''}(1- F_{R^{\theta_2}}(x)) \left(\nabla F_{R^{\theta_2}}(x)\right) \left(\nabla F_{R^{\theta_2}}(x)\right)\tr + h'(1 - F_{R^{\theta_2}}(x))\nabla^2 F_{R^{\theta_2}}(x)\Biggr|\Biggr|dx \\
        & \overset{(b)}{\leq} \int_{-M_r}^{M_r}\left\|h{''}(1- F_{R^{\theta_1}}(x)) \left(\nabla F_{R^{\theta_1}}(x)\right) \left(\nabla F_{R^{\theta_1}}(x)\right)\tr - h{''}(1- F_{R^{\theta_2}}(x)) \left(\nabla F_{R^{\theta_2}}(x)\right) \left(\nabla F_{R^{\theta_2}}(x)\right)\tr\right\| dx\\
        & \qquad \qquad \qquad+ \int_{-M_r}^{M_r} \left\|h'(1 - F_{R^{\theta_1}}(x))\nabla^2 F_{R^{\theta_1}}(x) - h'(1 - F_{R^{\theta_2}}(x))\nabla^2 F_{R^{\theta_2}}(x)\right\|  dx  \\
        & = \int_{-M_r}^{M_r}A_1  dx +  \int_{-M_r}^{M_r} A_2 dx,\numberthis\label{eq:a1a2} 
    \end{align*}
    where $A_1 = \left\|h{''}(1- F_{R^{\theta_1}}(x)) \left(\nabla F_{R^{\theta_1}}(x)\right) \left(\nabla F_{R^{\theta_1}}(x)\right)\tr - h{''}(1- F_{R^{\theta_2}}(x)) \left(\nabla F_{R^{\theta_2}}(x)\right) \left(\nabla F_{R^{\theta_2}}(x)\right)\tr\right\|$ and $A_2 = \left\|h'(1 - F_{R^{\theta_1}}(x))\nabla^2 F_{R^{\theta_1}}(x) - h'(1 - F_{R^{\theta_2}}(x))\nabla^2 F_{R^{\theta_2}}(x)\right\|.$ In the above, $(a)$ follows from Jensen's inequality applied to the norm function, and $(b)$ follows from the triangle inequality. 

We bound the term $A_1$ as follows: 
\begin{align*}
&
    \left\|h{''}(1- F_{R^{\theta_1}}(x)) \left(\nabla F_{R^{\theta_1}}(x)\right) \left(\nabla F_{R^{\theta_1}}(x)\right)\tr - h{''}(1- F_{R^{\theta_2}}(x)) \left(\nabla F_{R^{\theta_2}}(x)\right) \left(\nabla F_{R^{\theta_2}}(x)\right)\tr\right\|\\
    & = \Biggl|\Biggl|h{''}(1- F_{R^{\theta_1}}(x)) \left(\nabla F_{R^{\theta_1}}(x)\right) \left(\nabla F_{R^{\theta_1}}(x)\right)\tr -h{''}(1- F_{R^{\theta_1}}(x)) \left(\nabla F_{R^{\theta_1}}(x)\right) \left(\nabla F_{R^{\theta_2}}(x)\right)\tr \\
    & \qquad \qquad + h{''}(1- F_{R^{\theta_1}}(x)) \left(\nabla F_{R^{\theta_1}}(x)\right) \left(\nabla F_{R^{\theta_2}}(x)\right)\tr - h{''}(1- F_{R^{\theta_2}}(x)) \left(\nabla F_{R^{\theta_2}}(x)\right) \left(\nabla F_{R^{\theta_2}}(x)\right)\tr\Biggr|\Biggr|\\
    & \overset{(c)}{\leq} \Biggl|h{''}(1-F_{R^{\theta_1}}(x))\Biggr| \,\Biggl|\Biggl|\nabla F_{R^{\theta_1}}(x)\Biggr|\Biggr|\, \Biggl|\Biggl|\nabla F_{R^{\theta_1}}(x) - \nabla F_{R^{\theta_2}}(x)\Biggr|\Biggr| \\
    & \qquad \qquad \qquad + \Biggl|\Biggl|h{''}(1 - F_{R^{\theta_1}}(x))\nabla F_{R^{\theta_1}} - h{''}(1 - F_{R^{\theta_2}}(x))\nabla F_{R^{\theta_2}}\Biggr|\Biggr| \,\Biggl|\Biggl|\nabla F_{R^{\theta_2}}(x)\Biggr|\Biggr| \\
    & \overset{(d)}{\leq} M_{h{''}}T M_d(TM_h + T^2M_d^2)\Biggl|\Biggl|\theta_1 - \theta_2\Biggr|\Biggr| + T M_d \Biggl|\Biggl|h{''}(1 - F_{R^{\theta_1}}(x))\nabla F_{R^{\theta_1}}(x) - h{''}(1 - F_{R^{\theta_2}}(x))\nabla F_{R^{\theta_2}}(x)\Biggr|\Biggr|\\
    & = \xi^{'}_1 \Biggl|\Biggl|\theta_1 - \theta_2\Biggr|\Biggr| + T M_d \Biggl|\Biggl|h{''}(1 - F_{R^{\theta_1}}(x))\nabla F_{R^{\theta_1}}(x) - h{''}(1 - F_{R^{\theta_2}}(x))\nabla F_{R^{\theta_2}}(x)\Biggr|\Biggr|,
\end{align*} 
where $\xi^{'}_1 = M_{h{''}}T M_d(TM_h + T^2M_d^2).$ In above, $(c)$ follows from the triangle inequality of norm and Cauchy-Schwarz inequality and, $(d)$ follows by applying the mean-value theorem, \Cref{lem:cdfbounds}, and \Cref{ass:g_bound}. 

Notice that
\begin{align*}
&
    \Biggl|\Biggl|h{''}(1 - F_{R^{\theta_1}}(x))\nabla F_{R^{\theta_1}}(x) - h{''}(1 - F_{R^{\theta_2}}(x))\nabla F_{R^{\theta_2}}(x)\Biggr|\Biggr|\\
    & = \Biggl|\Biggl|h{''}(1 - F_{R^{\theta_1}}(x))\nabla F_{R^{\theta_1}}(x) -  h{''}(1 - F_{R^{\theta_1}}(x))\nabla F_{R^{\theta_2}}(x) \\
    & \qquad \qquad \qquad +  h{''}(1 - F_{R^{\theta_1}}(x))\nabla F_{R^{\theta_2}}(x) - h{''}(1 - F_{R^{\theta_2}}(x))\nabla F_{R^{\theta_2}}(x)\Biggr|\Biggr|\\
    & \overset{(e)}{\leq} \Biggl|h{''}(1 - F_{R^{\theta_1}}(x))\Biggr|\;\Biggl|\Biggl|\nabla F_{R^{\theta_1}}(x) - \nabla F_{R^{\theta_2}}(x)\Biggr|\Biggr| + \Biggl|\Biggl|\nabla_{F^{\theta_2}}(x)\Biggr|\Biggr| \; \Biggl|h{''}(1 - F_{R^{\theta_1}}(x)) - h{''}(1 - F_{R^{\theta_2}}(x))\Biggr| \\
    & \overset{(f)}{\leq} M_{h{''}}(T M_h + T^2M_d^2)\Biggl|\Biggl|\theta_1 - \theta_2\Biggr|\Biggr| + T M_d M_{h{'''}} \Biggl|F_{R^{\theta_1}}(x) - F_{R^{\theta_2}}(x)\Biggr|\\
    & \overset{(g)}{\leq} M_{h{''}}(T M_h + T^2M_d^2)\Biggl|\Biggl|\theta_1 - \theta_2\Biggr|\Biggr| + T^2 M_d^2 M_{h{'''}}\Biggl|\Biggl|\theta_1 - \theta_2\Biggr|\Biggr| \\ 
    & = \left(M_{h{''}}(T M_h + T^2M_d^2) + T^2 M_d^2 M_{h{'''}}\right)\Biggl|\Biggl|\theta_1 - \theta_2\Biggr|\Biggr| = \xi^{'}_2 \Biggl|\Biggl|\theta_1 - \theta_2\Biggr|\Biggr|,
\end{align*}
where $\xi^{'}_2 = \left(M_{h{''}}(T M_h + T^2M_d^2) + T^2 M_d^2 M_{h{'''}}\right).$ Further, $(e)$ follows from the properties of norm, $(f)$ follows from the mean-value theorem and \Cref{lem:cdfbounds}, and $(g)$ follows from the mean-value theorem in conjunction with \Cref{ass:g_bound}. 

Thus, 
\begin{align*}
&
    \int_{-M_r}^{M_r}A_1  dx \;\;\; \leq \;\;\; 2M_r (\xi^{'}_1 + T M_d \xi^{'}_2) \Biggl|\Bigg|\theta_1 - \theta_2 \Biggr|\Biggr| \;\;\; = \;\;\; \xi_1 \Biggl|\Bigg| \theta_1 - \theta_2\Biggr|\Biggr|,\numberthis\label{eq:a1bd}
\end{align*}
where $\xi_1 = 2M_r (\xi^{'}_1 + T M_d \xi^{'}_2).$

We now bound the term $A_2$ in \eqref{eq:a1a2}.
 \begin{align*}
    A_2&= \Biggl|\Bigg|h'(1 - F_{R^{\theta_1}}(x))\nabla^2 F_{R^{\theta_1}}(x) - g'(1 - F_{R^{\theta_2}}(x))\nabla^2 F_{R^{\theta_2}}(x)\Biggr|\Biggr|\\
     & = \Biggl|\Biggl|h'(1 - F_{R^{\theta_2}}(x))\nabla^2 F_{R^{\theta_2}}(x) - h'(1 - F_{R^{\theta_2}}(x))\nabla^2 F_{R^{\theta_1}}(x) \\
     & \qquad \qquad \qquad \qquad+ h'(1 - F_{R^{\theta_2}}(x))\nabla^2 F_{R^{\theta_1}}(x) - h'(1 - F_{R^{\theta_2}}(x))\nabla^2 F_{R^{\theta_2}}(x)\Biggr|\Biggr|\\
     & \overset{(h)}{\leq} \Biggl|\Biggl| h'(1 - F_{R^{\theta_2}}(x))\nabla^2 F_{R^{\theta_2}}(x) - h'(1 - F_{R^{\theta_2}}(x))\nabla^2 F_{R^{\theta_1}}(x)\Biggr|\Biggr| \\
     & \qquad \qquad \qquad \qquad + \Biggl|\Biggl|h'(1 - F_{R^{\theta_2}}(x))\nabla^2 F_{R^{\theta_1}}(x) - h'(1 - F_{R^{\theta_2}}(x))\nabla^2 F_{R^{\theta_2}}(x)\Biggr|\Biggr|\\
     & \overset{(i)}{=} \Biggl|h'(1 - F_{R^{\theta_2}}(x))\Biggr| \; \Biggl|\Bigg|\nabla^2 F_{R^{\theta_2}}(x) - \nabla^2 F_{R^{\theta_1}}(x)\Biggr|\Biggr| + \Biggl|h'(1 - F_{R^{\theta_2}}(x)) - h'(1 - F_{R^{\theta_1}}(x))\Biggr|\; \Biggl|\Biggl|\nabla^2 F_{R^{\theta_1}}\Biggr|\Biggr| \\
     & \overset{(j)}{\leq}  \Biggl|h'(1 - F_{R^{\theta_2}}(x))\Biggr| \; \Biggl|\Biggl|\nabla^2 F_{R^{\theta_2}}(x) - \nabla^2 F_{R^{\theta_1}}(x)\Biggr|\Biggr| + M_{h{''}}\Biggl|F_{R^{\theta_1}}(x) - F_{R^{\theta_2}}(x)\Biggr| \; \Biggl|\Biggl|\nabla^2 F_{R^{\theta_1}}\Biggr|\Biggr| \\
     & \overset{(k)}{\leq} \Biggl|h'(1 - F_{R^{\theta_2}}(x))\Biggr| \Biggl|\Biggl|\nabla^2 F_{R^{\theta_2}}(x) - \nabla^2 F_{R^{\theta_1}}(x) \Biggr|\Biggr| + M_{h{''}}T M_d (T M_h + T^2 M_d^2) \Biggl|\Biggl|\theta_1 - \theta_2\Biggr|\Biggr| \\
     & \overset{(l)}{\leq}  M_{h'} \Biggl|\Biggl|\nabla^2 F_{R^{\theta_2}}(x) - \nabla^2 F_{R^{\theta_1}}(x) \Biggr|\Biggr| +  M_{h{''}}T M_d (T M_h + T^2 M_d^2) \Biggl|\Biggl|\theta_1 - \theta_2\Biggr|\Biggr| \\
     & = M_{h'} \Biggl|\Biggl|\nabla^2 F_{R^{\theta_2}}(x) - \nabla^2 F_{R^{\theta_1}}(x) \Biggr|\Biggr| +  \xi^{'}_3  \Biggl|\Biggl|\theta_1 - \theta_2\Biggr|\Biggr|, 
 \end{align*}
 where $\xi^{'}_3 = M_{h{''}}T M_d (T M_h + T^2 M_d^2).$ In the above, $(h)$ and $(i)$ follow from the properties of norm, $(j)$ and $(k)$ follows from mean-value theorem, $(l)$ follows from \Cref{ass:g_bound}. 

Let $\Phi_{\theta_i}^t = \nabla \log \pi_{\theta_i}(A_t(\omega)|S_t(\omega))$.
Note that 
\begin{align*}
&
    \Biggl|\Biggl|\nabla^2 F_{R^{\theta_2}}(x) - \nabla^2 F_{R^{\theta_1}}(x) \Biggr|\Biggr|\\
    &= \Biggl|\Biggl|\sum_{\omega \in \Omega} \mathbf{1} \{R(\omega) \leq x\}\sum_{t=0}^{T-1} \left(\nabla^2 \log \pi_{\theta_2}(A_t(\omega)|S_t(\omega))\mathbb{P_{\theta}}(\omega) - \nabla^2 \log \pi_{\theta_1}(A_t(\omega)|S_t(\omega))\mathbb{P_{\theta}}(\omega) \right) \\
      &\qquad \qquad   + \sum_{\omega \in \Omega} \mathbf{1} \{R(\omega) \leq x\}\sum_{t=0}^{T-1} \left((\Phi_{\theta_2}^t) (\Phi_{\theta_2}^t)\tr - (\Phi_{\theta_1}^t) (\Phi_{\theta_1}^t)\tr \right) \mathbb{P}_{\theta}(\omega) \Biggr|\Biggr| \\
      & \overset{(m)}{\leq} T L_2 \Lambda \Biggl|\Biggl|\theta_1 - \theta_2\Biggr|\Biggr| + T \Lambda \left(\Biggl|\Bigg|\Phi_{\theta_2}^t - \Phi_{\theta_1}^t\Biggr|\Biggr|\right) \Biggl(\Biggl|\Biggl|\Phi_{\theta_2}^t\Biggr|\Biggr| + \Biggl|\Biggl|\Phi_{\theta_1}^t\Biggr|\Biggr|\Biggr)\\
      & \overset{(o)}{\leq} T L_2 \Lambda \Biggl|\Biggl|\theta_1 - \theta_2\Biggr|\Biggr| + 2 T M_d M_h \Lambda \Biggl|\Biggl|\theta_1 - \theta_2\Biggr|\Biggr| \\ 
      & = (T L_2 + 2 TM_d M_h) \Lambda \Biggl|\Biggl|\theta_1 - \theta_2\Biggr|\Biggr|\\
      & = \xi^{'}_4\Biggl|\Biggl|\theta_1 - \theta_2\Biggr|\Biggr|, 
\end{align*}
where $\xi^{'}_4 = (T L_2 + 2 TM_d M_h) \Lambda.$ In the above, $(m)$ follows from \Cref{ass:lip_hessian} and $\|XX\tr - YY\tr\| \leq (\|X - Y\|)(\|X\|+ \|Y\|)$, and $(o)$ follows from \Cref{lem:cdfbounds}, mean-value theorem, and \Cref{ass:Hessian_bound}.  

Thus, we have
\begin{align*}
    \int_{-M_r}^{M_r} A_2 dx \;\;\;\leq\;\;\; 2M_r(M_{h'}\xi^{'}_4 + \xi^{'}_3) \Biggl|\Biggl|\theta_1 - \theta_2\Biggr|\Biggr| \;\;\; = \;\;\; \xi_2 \Biggl|\Biggl|\theta_1 - \theta_2\Biggr|\Biggr|,\numberthis\label{eq:a2bd}
\end{align*}
where $\xi_2 = 2M_r(M_{h'}\xi^{'}_4 + \xi^{'}_3).$ 

Substituting \eqref{eq:a1bd} and \eqref{eq:a2bd} in \eqref{eq:a1a2}, we obtain
\[\Biggl|\Biggl|\nabla^2 \rho_h(\theta_1) - \nabla^2 \rho_h(\theta_2)\Biggr|\Biggr| \leq L_{\mathcal{H}} \Biggl|\Biggl|\theta_1 - \theta_2\Biggr|\Biggr|,\] where $L_{\mathcal{H}} = \xi_1 + \xi_2.$ 

 Hence proved. 
\end{proof}    
\section{Proofs of the claims in \Cref{sec:DRM policy gradient and Hessian estimation,sec:CRPN_drm}}  \label{sec:estimation_appendix}
\subsection{Efficient computation of DRM gradient and Hessian estimates}
We claimed in \Cref{sec:DRM policy gradient and Hessian estimation} that the gradient and Hessian estimates defined in \eqref{eq:drm_grad_estimation} and \eqref{eq:drm_Hessian_estimation}, respectively, can be computed efficiently using the order statistics of the $b$ samples $ \{R_i^{\theta}\}_{i = 1}^b$. The results below make this claim precise.
\begin{lemma}\label{lemma: drm_grad_estimator} 
Let $R_{(i)}^{\theta}$ be the $i$-th smallest order statistic from the samples $\{R_{(i)}^{\theta}\}_{i=0}^{m}$. Then,
    \begin{align*}
    &
    \widehat{\nabla} \rho_h (\theta) = \frac{1}{m}\sum_{i=1}^{m-1} \Biggl( R^{\theta}_{(i)} - R^{\theta}_{(i+1)}\Biggr) h'\Biggl(1 - \frac{i}{m}\Biggr)\sum_{j = 1}^{i} \nabla l^{\theta}_{(j)}   + \frac{1}{m} \Biggl(R^{\theta}_{(m)} - M_r\Biggr) h'_{+} (0) \sum_{j = 1}^{m} \nabla l^{\theta}_{(j)}.
    \end{align*}
\end{lemma}
\begin{proof}
    See \cite[Lemma 3]{vijayan2021policy}. 
\end{proof}

\begin{lemma}\label{lemma: drm_sample_estimator}
Let $R_{(i)}^{\theta}$ be the $i$-th smallest order statistic from the samples $\{R_{(i)}^{\theta}\}_{i=0}^{b}$. Then we have
    \begin{align*}
    &
        \widehat{\nabla}^2 \rho_h (\theta)= \frac{1}{b^2}\sum_{i=1}^{b-1} \Biggl[\left[R_{(i+1)}^{\theta} \!-\! R_{(i)}^{\theta}\right] h{''}\left[1 - \frac{i}{b}\right]\!\left[\sum_{j = 1}^i \nabla l_{(j)}^{\theta}\right]\! \left[\sum_{j = 1}^i \nabla l_{(j)}^{\theta}\right]\tr\Biggr] \\
    &   - \frac{1}{b}\sum_{i=1}^{b-1}\left(R_{(i+1)}^{\theta} - R_{(i)}^{\theta}\right) h'\left(1 - \frac{i}{b}\right) \left(\sum_{j = 1}^i \nabla^2 l_{(j)}^{\theta} + \sum_{j = 1}^i \left(\nabla l_{(j)}^{\theta}\right) \left(\nabla l_{(j)}^{\theta}\right)\tr \right) \\ 
    & \qquad \qquad+ \frac{1}{b^2} \left(M_r - R_{(b)}^{\theta}\right)\left(h_{+}^{''}(0) \left(\sum_{j = 1}^i \nabla l_{(j)}^{\theta}\right) \left(\sum_{j = 1}^i \nabla l_{(j)}^{\theta}\right)\tr\right)\\
    &\qquad \qquad-\frac{1}{b}\Biggl[ \left(M_r - R_{(b)}^{\theta}\right)h'_{+}(0) \left(\sum_{j = 1}^b \nabla^2 l_{(j)}^{\theta} + \sum_{j = 1}^b \nabla l_{(j)}^{\theta} {\nabla l_{(j)}^{\theta}}\tr\right)\Biggr], 
    \end{align*}
    where $\nabla l_{(i)}^{\theta} = \sum_{t = 0}^{T - 1} \nabla \log \pi_{\theta}(A^i_t| S^i_t).$
\end{lemma} 

\begin{proof}
\label{proof: drm_sample_estimator}
 We have the following estimates 
 from \cite{vijayan2021policy}. 
    \begin{equation}
    \label{eq:eq_edf}
    G_{R^{\theta}}^m (x) = 
        \begin{cases}
            0, & \ \text{if} \ x < R_{(1)}^{\theta}\\
            \frac{i}{m}, & \ \text{if} \ R_{(i)}^{\theta} \leq x < R_{(i+1)}^{\theta},\\
             & i \in \{1,\ldots, m-1\}\\
            1, & \ \text{if} \ x \geq R_{(m)}^{\theta}. 
        \end{cases}
    \end{equation}
    and 
    \begin{equation}
    \label{eq:eq_grad}
    \widehat{\nabla} G_{R^{\theta}}^m (x) = 
        \begin{cases}
            0, & \ \text{if} \ x < R_{(1)}^{\theta}\\
            \frac{1}{m}\sum_{j = 1}^i \nabla l_{(j)}^{\theta}, & \ \text{if} \ R_{(i)}^{\theta} \leq x < R_{(i+1)}^{\theta},\\
            & i \in \{1,\ldots, m-1\}\\
            \frac{1}{m}\sum_{j = 1}^m \nabla l_{(j)}^{\theta}, & \ \text{if} \ x \geq R_{(m)}^{\theta}. 
        \end{cases}
    \end{equation}
We form the estimate  $\widehat{\nabla}^2 G_{R^{\theta}}^m (x)$ of $\nabla^2 G_{R^{\theta}}^m (x)$ as follows: 
\begin{equation} 
\label{eq:eq_hes}
    \widehat{\nabla}^2 G_{R^{\theta}}^m (x) = 
        \begin{cases}
            0, & \ \text{if} \ x < R_{(1)}^{\theta}\\
            \frac{1}{m}\sum_{j = 1}^i \nabla^2 l_{(j)}^{\theta} + \frac{1}{m}\sum_{j = 1}^i \left(\nabla l_{(j)}^{\theta}\right) \left(\nabla l_{(j)}^{\theta}\right)\tr, & \ \text{if} \ R_{(i)}^{\theta} \leq x < R_{(i+1)}^{\theta},\\
            & i \in \{1,\ldots, m-1\}\\
            \frac{1}{m}\sum_{j = 1}^m \nabla^2 l_{(j)}^{\theta} + \frac{1}{m}\sum_{j = 1}^m \left(\nabla l_{(j)}^{\theta}\right) \left(\nabla l_{(j)}^{\theta}\right)\tr, & \ \text{if} \ x \geq R_{(m)}^{\theta}. 
        \end{cases}
    \end{equation}

    Using \eqref{eq:eq_edf}, \eqref{eq:eq_grad}, and \eqref{eq:eq_hes}, we obtain 
    \begin{align*}
         &
    \widehat{\nabla}^2 \rho_h (\theta) \\
    & = \int_{-M_r}^{M_r} \left(h{''}(1 - G^m_{R^{\theta}}(x)) \left(\widehat{\nabla} G^m_{R^{\theta}}(x)\right) \left(\widehat{\nabla} G^m_{R^{\theta}}(x)\right)\tr -  h'(1 - G^m_{R^{\theta}}(x)) \widehat{\nabla}^2 G^m_{R^{\theta}}(x)\right) dx \\
    & = \int_{-M_r}^{R_{(1)}^{\theta}} \left(h{''}(1 - G^m_{R^{\theta}}(x)) \left(\widehat{\nabla} G^m_{R^{\theta}}(x)\right) \left(\widehat{\nabla} G^m_{R^{\theta}}(x)\right)\tr -  h'(1 - G^m_{R^{\theta}}(x)) \widehat{\nabla}^2 G^m_{R^{\theta}}(x)\right) dx\\
    &\qquad + \sum_{i=1}^{m-1}\int_{R_{(i)}^{\theta}}^{R_{(i+1)}^{\theta}} \left(h{''}(1 - G^m_{R^{\theta}}(x)) \left(\widehat{\nabla} G^m_{R^{\theta}}(x)\right) \left(\widehat{\nabla} G^m_{R^{\theta}}(x)\right)\tr -  h'(1 - G^m_{R^{\theta}}(x)) \widehat{\nabla}^2 G^m_{R^{\theta}}(x)\right) dx\\
    &\qquad + \int_{R_{(m)}^{\theta}}^{M_r} \left(h{''}(1 - G^m_{R^{\theta}}(x)) \left(\widehat{\nabla} G^m_{R^{\theta}}(x)\right) \left(\widehat{\nabla} G^m_{R^{\theta}}(x)\right)\tr -  h'(1 - G^m_{R^{\theta}}(x)) \widehat{\nabla}^2 G^m_{R^{\theta}}(x)\right) dx \\
    & = \frac{1}{m}\sum_{i=1}^{m-1}\int_{R_{(i)}^{\theta}}^{R_{(i+1)}^{\theta}} h{''}(1 - \frac{i}{m}) \left(\sum_{j = 1}^i \nabla l_{(j)}^{\theta}\right) \left(\sum_{j = 1}^i \nabla l_{(j)}^{\theta}\right)\tr  dx\\
    & \qquad - \frac{1}{m}\sum_{i=1}^{m-1}\int_{R_{(i)}^{\theta}}^{R_{(i+1)}^{\theta}} g'(1 - \frac{i}{m}) \left(\sum_{j = 1}^i \nabla^2 l_{(j)}^{\theta} + \sum_{j = 1}^i \left(\nabla l_{(j)}^{\theta}\right) \left(\nabla l_{(j)}^{\theta}\right)\tr\right) dx \\ 
    &\qquad + \frac{1}{m}\int_{R_{(m)}^{\theta}}^{M_r} \left(h_{+}^{''}(0) \left(\sum_{j = 1}^i \nabla l_{(j)}^{\theta}\right) \left(\sum_{j = 1}^i \nabla l_{(j)}^{\theta}\right)\tr\right) dx\\
    & \qquad - \frac{1}{m}\int_{R_{(m)}^{\theta}}^{M_r} h'_{+}(0) \left(\sum_{j = 1}^m \nabla^2 l_{(j)}^{\theta} + \sum_{j = 1}^m \left(\nabla l_{(j)}^{\theta}\right) \left(\nabla l_{(j)}^{\theta}\right)\tr\right) dx \\
    & = \frac{1}{m}\sum_{i=1}^{m-1} \left(R_{(i+1)}^{\theta} - R_{(i)}^{\theta}\right) g{''}(1 - \frac{i}{m}) \left(\sum_{j = 1}^i \nabla l_{(j)}^{\theta}\right) \left(\sum_{j = 1}^i \nabla l_{(j)}^{\theta}\right)\tr \\
    & \qquad - \frac{1}{m}\sum_{i=1}^{m-1}\left(R_{(i+1)}^{\theta} - R_{(i)}^{\theta}\right) h'(1 - \frac{i}{m}) \left(\sum_{j = 1}^i \nabla^2 l_{(j)}^{\theta} + \sum_{j = 1}^i \left(\nabla l_{(j)}^{\theta}\right) \left(\nabla l_{(j)}^{\theta}\right)\tr\right) \\ 
    &\qquad + \frac{1}{m} \left(M_r - R_{(m)}^{\theta}\right)\left(h_{+}^{''}(0) \left(\sum_{j = 1}^m \nabla l_{(j)}^{\theta}\right) \left(\sum_{j = 1}^m \nabla l_{(j)}^{\theta}\right)\tr\right)\\
    & \qquad - \frac{1}{m} \left(M_r - R_{(m)}^{\theta}\right)h'_{+}(0) \left(\sum_{j = 1}^m \nabla^2 l_{(j)}^{\theta} + \sum_{j = 1}^m \left(\nabla l_{(j)}^{\theta}\right) \left(\nabla l_{(j)}^{\theta}\right)\tr\right). 
    \end{align*}
\end{proof} 
\subsection{Proof of Lemma \ref{lemma: drm_sample_estimator_alt}} 
\begin{proof}
\label{proof: drm_sample_estimator_alt}
From Lemma \eqref{lemma: drm_sample_estimator}, we can define the following two terms such that $\widehat{\nabla}^2 \rho_h (\theta) := \widehat{\nabla}^2 \rho^{A}_g (\theta) - \widehat{\nabla}^2 \rho^{B}_h (\theta)$, where
\begin{align*}
\widehat{\nabla}^2 \rho^{A}_h (\theta) &:= \frac{1}{m}\sum_{i=1}^{m-1} \left(R_{(i+1)}^{\theta} - R_{(i)}^{\theta}\right) h{''}(1 - \frac{i}{m}) \left(\sum_{j = 1}^i \nabla l_{(j)}^{\theta}\right) \left(\sum_{j = 1}^i \nabla l_{(j)}^{\theta}\right)\tr\\
&+ \frac{1}{m} \left(M_r - R_{(m)}^{\theta}\right) h_{+}^{''}(0) \left(\sum_{j = 1}^i \nabla l_{(j)}^{\theta}\right) \left(\sum_{j = 1}^i \nabla l_{(j)}^{\theta}\right)\tr, \\
\widehat{\nabla}^2 \rho^{B}_h (\theta) &:= \frac{1}{m}\sum_{i=1}^{m-1}\left(R_{(i+1)}^{\theta} - R_{(i)}^{\theta}\right) h'(1 - \frac{i}{m}) \left(\sum_{j = 1}^i \nabla^2 l_{(j)}^{\theta} + \sum_{j = 1}^i \left(\nabla l_{(j)}^{\theta}\right) \left(\nabla l_{(j)}^{\theta}\right)\tr\right) \\
&+ \frac{1}{m} \left(M_r - R_{(m)}^{\theta}\right)h'_{+}(0) \left(\sum_{j = 1}^m \nabla^2 l_{(j)}^{\theta} + \sum_{j = 1}^m \left(\nabla l_{(j)}^{\theta}\right) \left(\nabla l_{(j)}^{\theta}\right)\tr\right).
\end{align*}
Given the construction of $c''_i$ and $c'_i$ from Lemma \eqref{lemma: drm_sample_estimator_alt} we can extend the upper limit of the summation to $m$ for both the terms. Re-writing the two terms as follows:
\begin{align*}
\widehat{\nabla}^2 \rho^{A}_h (\theta) &= \frac{1}{m}\sum_{i=1}^{m}  c''_i \left(\sum_{j = 1}^i \nabla l_{(j)}^{\theta}\right) \left(\sum_{j = 1}^i \nabla l_{(j)}^{\theta}\right)\tr & \widehat{\nabla}^2 \rho^{B}_h (\theta) &= \frac{1}{m}\sum_{i=1}^{m} c'_i \left(\sum_{j = 1}^i \nabla^2 l_{(j)}^{\theta} + \sum_{j = 1}^i \left(\nabla l_{(j)}^{\theta}\right) \left(\nabla l_{(j)}^{\theta}\right)\tr\right)\\
&= \frac{1}{m}\sum_{i=1}^{m}  c''_i \left(\nabla \sum_{j = 1}^i l_{(j)}^{\theta}\right) \left(\nabla \sum_{j = 1}^i  l_{(j)}^{\theta}\right)\tr & &= \frac{1}{m}\sum_{i=1}^{m} c'_i \sum_{j = 1}^i \left(\nabla^2 l_{(j)}^{\theta} + \left(\nabla l_{(j)}^{\theta}\right) \left(\nabla l_{(j)}^{\theta}\right)\tr\right)\\
&= \frac{1}{m}\sum_{i=1}^{m}  c''_i \nabla s^{\theta}_i \nabla\tr s^{\theta}_i, \textrm{ and } & &= \frac{1}{m} \sum_{j=1}^{m} \left(\nabla^2 l_{(j)}^{\theta} + \left(\nabla l_{(j)}^{\theta}\right) \left(\nabla l_{(j)}^{\theta}\right)\tr\right) \left(\sum_{i=j}^{m} c'_i \right)\\
& & &= \frac{1}{m} \sum_{j=1}^{m} \psi_j \left(\nabla^2 l_{(j)}^{\theta} + \nabla l_{(j)}^{\theta} \nabla \tr l_{(j)}^{\theta} \right),
\end{align*}
where $s^{\theta}_i := \sum_{j = 1}^i  l_{(j)}^{\theta}$ and $\psi_j = \sum_{i=j}^{m} c'_i$. In the third line for term $\widehat{\nabla}^2 \rho^{B}_h (\theta)$, we switch the order of summations as $\sum_{i=1}^m \sum_{j=1}^i \equiv \sum_{j=1}^m \sum_{i=j}^m$. Combining these two terms, we get the final form of the Hessian estimate given in Lemma \eqref{lemma: drm_sample_estimator_alt}. The expression for $\widehat{\nabla} \rho_g (\theta)$ is obtained using similar technique as demonstrated for $\widehat{\nabla}^2 \rho^B_g (\theta)$. Re-writing the Hessian estimate in this form, we can adopt the computation trick used in \cite{maniyar2024crpn}.

\paragraph{Variance-reduced DRM estimates.}
We first argue that any cross-trajectory term of the form $r(\tau_i) \nabla \log p(\tau_j)$ is zero in expectation over all possible trajectories if $ i\neq j$. This holds true as $\E_{\tau}{[\nabla  \log p(\tau)]} = \E_{\tau}{[\frac{\nabla p(\tau)}{p(\tau)}]} = \sum_{\tau} \frac{\nabla p(\tau)}{p(\tau)} p(\tau) = \sum_{\tau} \nabla p(\tau) = \nabla \sum_{\tau} p(\tau) = \nabla [1] = 0$. Furthermore, $\nabla  \E_{\tau}{[\nabla  \log p(\tau)]} = \sum_{\tau} p(\tau) \nabla^2 \log p(\tau) + p(\tau) \nabla \log p(\tau) \nabla^T \log p(\tau) = \E_{\tau}{[\nabla^2 \log p(\tau) + \nabla \log p(\tau) \nabla^T \log p(\tau)]} = 0$. Now, we consider the gradient estimate
\begin{align*}
    \widehat{\nabla} \rho_h &= -\frac{1}{m} \sum_{i=1}^m \psi'_{i} \nabla l^{\theta}_{(i)} = -\frac{1}{m} \sum_{i=1}^m \sum_{j=i}^m c'_j \nabla l^{\theta}_{(i)} = -\frac{1}{m} \sum_{i=1}^m \left( c'_i \nabla l^{\theta}_{(i)} + \sum_{j=i+1}^m c'_j \nabla l^{\theta}_{(i)} \right) , \\
    &= \frac{1}{m} \sum_{i=1}^m \left( R^{\theta}_{(i)} h'(1-\frac{i}{m}) \nabla l^{\theta}_{(i)} \right)
    - \frac{1}{m} \sum_{i=1}^m \left( R^{\theta}_{(i+1)} h'(1-\frac{i}{m}) \nabla l^{\theta}_{(i)} \right) 
    - \frac{1}{m} \sum_{i=1}^m \sum_{j=i+1}^m c'_j \nabla l^{\theta}_{(i)} .
\end{align*}
Notice that the second and third term are cross-trajectory terms with zero expectation and thus can be ignored in order to reduce variance in our DRM estimates. For term A in our DRM Hessian estimate, we re-write as follows
\begin{align*}
    \widehat{\nabla}^2 \rho^{A}_h(\theta) &=  \frac{1}{m}\sum_{i=1}^{m}  c''_i \left(\sum_{j = 1}^i \nabla 
 l_{(j)}^{\theta}\right) \left(\sum_{j = 1}^i \nabla l_{(j)}^{\theta}\right)\tr = \frac{1}{m}\sum_{i=1}^{m}  c''_i \left( \sum_{j=1}^{i} \nabla l_{(j)}^{\theta} \nabla \tr l_{(j)}^{\theta} + \sum_{\substack{j, k=1\\ j \neq k}}^{i} \nabla l_{(j)}^{\theta} \nabla^T l_{(k)}^{\theta} \right) \\
 &= \frac{1}{m} \sum_{i=1}^{m}  \left( \sum_{j=i}^{m} c''_j \right) \sum_{j=1}^{i} \nabla l_{(j)}^{\theta} \nabla \tr l_{(j)}^{\theta} + \frac{1}{m}\sum_{i=1}^{m}  c''_i \sum_{\substack{j, k=1\\ j \neq k}}^{i} \nabla l_{(j)}^{\theta} \nabla \tr l_{(k)}^{\theta},
\end{align*}
with the second term having zero expectation over trajectories $j \neq k$ as $\E_{\tau} [\nabla \log p(\tau)] = 0$. The argument for term B follows in a similar fashion to that of the gradient estimate, but due to $\E_{\tau} [\nabla^2 \log p(\tau) + \nabla  \log p(\tau) \nabla^T \log p(\tau)]$ being 0. The authors of \cite{markowitz2023} argue that by removing such cross-trajectory terms, we end up reducing the variance of our unbiased estimates. These estimates have also shown better performance empirically.

\end{proof}

\section{Proofs of the claims in \Cref{sec:main_results}}
\label{sec:main_appendix} 
In this section, we prove \Cref{lemma: error_bound_Hessian}  and \Cref{thm: convergence_SOSP}. Now we state an useful results from \cite{vijayan2021policy} below, which will be used in the proof of  \Cref{lemma: error_bound_Hessian}. 

\subsection{Proof of Lemma \ref{lemma: error_bound_Hessian}}
\label{proof: error_bound_Hessian}
We state and prove a few intermediate results that will be used in the proof of Lemma \ref{lemma: error_bound_Hessian}.

\begin{proposition}
\label{propo_grad2}
     For all $x \in (-M_r, M_r)$, 
     \begin{align*}
    \E\left [ \left|G^{m}_{R^{\theta}}(x) - F_{R^{\theta}}(x) \right|^2\right] &\leq \frac{4}{m}. \\
    \E\left[\left\|\widehat{\nabla} G^{m}_{R^{\theta}}(x) - \nabla F_{R^{\theta}}(x)\right\|^2\right] &\leq \frac{4e^2 T^2 M_d^2}{m}.          
     \end{align*}
\end{proposition}
\begin{proof}
    These identities can be extracted from the proof of  Lemma 4 in \cite{vijayan2021policy}. 
\end{proof}

\begin{lemma} 
\label{lemma:error_bound_gradient}
Let the gradient estimate $\widehat{\nabla} \rho_h (\theta)$  be computed by \eqref{eq:drm_grad_estimation} with $m$ number of trajectories. From \cite{vijayan2021policy}, we have
    \begin{align}
    \label{eq:error_bound_grad}
    \E \left[ \left\|\widehat{\nabla} \rho_h (\theta) - \nabla \rho_h (\theta)\right \|^2\right] \leq \frac{\kappa_1}{m},
    \end{align}
    where $\kappa_1 = 32M_r^2 T^2 M_d^2 (e^2 M_{h'}^2 + M_{h''}^2). $
\end{lemma}
\begin{proof}
     See \cite[Lemma 4]{vijayan2021policy}. 
\end{proof}

\begin{proposition}
\label{propo_grad4}
     For all $x \in (-M_r, M_r)$, 
    \[\E\left [ \left|G^{m}_{R^{\theta}}(x) - F_{R^{\theta}}(x) \right|^4\right] \leq \frac{16}{m^2}.\]
\end{proposition}
\begin{proof}
    Using Azuma-Hoeffding inequality, we have
    \begin{align*}
        P \Bigl(|G^m_{R^{\theta}}(x) - F_{R^{\theta}}(x)| > \epsilon\Bigr) \leq 2 e^{-\frac{m\epsilon^2}{2}}. 
    \end{align*}
    Thus,
    \begin{align*}
        \E \left[|G^m_{R^{\theta}}(x) - F_{R^{\theta}}(x)|^4 > \epsilon\right] & = \int\limits_{0}^{\infty} P \Bigl(|G^m_{R^{\theta}}(x) - F_{R^{\theta}}(x)| > \epsilon^{\frac{1}{4}}\Bigr) d\epsilon\\
        & \leq \int_{0}^{\infty} 2 e^{\frac{-m\sqrt{\epsilon}}{2}}d\epsilon = \frac{16}{m^2}. 
    \end{align*}
\end{proof}
\begin{proposition}
\label{propo_Hessian4} 
    For all $x \in (-M_r, M_r)$,
    \[\E\left[\left\|\widehat{\nabla} G^{m}_{R^{\theta}}(x) - \nabla F_{R^{\theta}}(x)\right\|^4\right] \leq \frac{16e^2 T^4 M_d^4}{m^2}. \]
\end{proposition}
\begin{proof}
    Using Azuma-Hoeffding inequality, we have
    \begin{align*}
        P \Bigl(\left \|\nabla G^m_{R^{\theta}}(x) - \nabla F_{R^{\theta}}(x)\right \| > \epsilon\Bigr) \leq 2e^2 e^{-\frac{m\epsilon^2}{2T^2 M^2_d}}. 
    \end{align*}
    Thus,
    \begin{align*}
        \E \left[\|\nabla G^m_{R^{\theta}}(x) - \nabla F_{R^{\theta}}(x)\|^4 > \epsilon\right] & = \int\limits_{0}^{\infty} P \Bigl(\|\nabla G^m_{R^{\theta}}(x) - \nabla F_{R^{\theta}}(x)\| > \epsilon^{\frac{1}{4}}\Bigr) d\epsilon\\
        & \leq 2e^2\int_{0}^{\infty} e^{\frac{-m\sqrt{\epsilon}}{2T^2 M^2_d}}d\epsilon = \frac{16e^2T^4 M^4_d}{m^2}. 
    \end{align*}
\end{proof}

\begin{proof}\textbf{\textit{(Lemma \ref{lemma: error_bound_Hessian})}}\ \\
Notice that
\begin{align*}
     &
     \E \left[\left\| \widehat{\nabla}^2 \rho_h(\theta) - \nabla^2 \rho_h(\theta)\right\|^2\right] \\
     & = \E \Biggl[ \Biggl| \Biggl|\int_{-M_r}^{M_r} h{''}(1 - G^m_{R^{\theta}}(x)) \left(\widehat{\nabla} G^m_{R^{\theta}}(x)\right) \left(\widehat{\nabla} G^m_{R^{\theta}}(x)\right)\tr dx - \int_{-M_r}^{M_r} h'(1 - G^m_{R^{\theta}}(x)) \widehat{\nabla}^2 G^m_{R^{\theta}}(x) dx. \\
     & \qquad \qquad -  \int_{-M_r}^{M_r} h{''}(1 - F_{R^{\theta}}(x)) \left(\nabla F_{R^{\theta}}(x)\right)\left(\nabla F_{R^{\theta}}(x)\right)\tr dx + \int_{-M_r}^{M_r} h'(1 - F_{R^{\theta}}(x))\nabla^2 F_{R^{\theta}}(x)dx \Biggr| \Biggr|^2 \Biggr]\\
     &  \overset{(a)}{\leq} 2 \E \Biggl[ \Biggl| \Biggl|\int_{-M_r}^{M_r} \left[h{''}(1 - G^m_{R^{\theta}}(x)) \left(\widehat{\nabla} G^m_{R^{\theta}}(x)\right) \left(\widehat{\nabla} G^m_{R^{\theta}}(x)\right)\tr -  h{''}(1 - F_{R^{\theta}}(x)) \left(\nabla F_{R^{\theta}}(x)\right)\left(\nabla F_{R^{\theta}}(x)\right)\tr\right] dx \Biggr|\Biggr|^2 \\
     & \qquad \qquad \qquad  + \Biggl|\Bigg|\int_{-M_r}^{M_r} \left[h'(1 - F_{R^{\theta}}(x))\nabla^2 F_{R^{\theta}}(x) - h'(1 - G^m_{R^{\theta}}(x)) \widehat{\nabla}^2 G^m_{R^{\theta}}(x)\right] dx \Biggr|\Bigg|^2 \Biggr]\\
     & = 2 \E \Biggl[ \Biggl| \Biggl|\int_{-M_r}^{M_r} \left[h{''}(1 - G) \left(\widehat{\nabla} G\right) \left(\widehat{\nabla} G\right)\tr -  h{''}(1 - F) \left(\nabla F\right)\left(\nabla F\right)\tr\right] dx \Biggr|\Biggr|^2\\
     & \qquad \qquad \qquad + \Biggl|\Bigg|\int_{-M_r}^{M_r} \left[h'(1 - F)\nabla^2 F - h'(1 - G) \widehat{\nabla}^2 G \right]dx \Biggr|\Bigg|^2 \Biggr]\\
     & = 2 \E \Biggl[ \Biggl| \Biggl|\int_{-M_r}^{M_r} D_1 dx \Biggr|\Biggr|^2 + \Biggl|\Bigg|\int_{-M_r}^{M_r} D_2 dx \Biggr|\Bigg|^2 \Biggr].\numberthis \label{eq:D1D2}
\end{align*}
In the above, for notational simplicity, we set $G = G^m_{R^{\theta}}(x), \widehat{\nabla} G = \widehat{\nabla} G^m_{R^{\theta}}(x), F = F_{R^{\theta}}(x), \nabla F = \nabla F_{R^{\theta}}(x), \nabla^2 F = \nabla^2 F_{R^{\theta}}(x)$,  $\widehat{\nabla}^2 G = \widehat{\nabla}^2 G^m_{R^{\theta}}(x),D_1 = h{''}(1 - G) \left(\widehat{\nabla} G\right) \left(\widehat{\nabla} G\right)\tr -  h{''}(1 - F) \left(\nabla F\right)\left(\nabla F\right)\tr,$ and $D_2 = h'(1 - F)\nabla^2 F - h'(1 - G) \widehat{\nabla}^2 G.$ Further, $(a)$ follows from the fact that $\|X + Y\|^2 \leq 2(\|X\|^2 + \|Y\|^2)$. 

We now bound the term $\Biggl| \Biggl|\int_{-M_r}^{M_r} D_1 dx \Biggr|\Biggr|^2$ in \eqref{eq:D1D2} as follows:
\begin{align*}
    \Biggl| \Biggl|\int_{-M_r}^{M_r} D_1 dx \Biggr|\Biggr|^2
    & = \Biggl| \Biggl|\int_{-M_r}^{M_r} h{''}(1 - G) \left(\widehat{\nabla} G\right) \left(\widehat{\nabla} G\right)\tr -  h{''}(1 - F) \left(\nabla F\right)\left(\nabla F\right)\tr dx \Biggr|\Biggr|^2 \\
    & \overset{(b)}{\leq} 2M_r\int_{-M_r}^{M_r} \Biggl| \Biggl| h{''}(1 - G) \left(\widehat{\nabla} G\right) \left(\widehat{\nabla} G\right)\tr -  h{''}(1 - F) \left(\nabla F\right)\left(\nabla F\right)\tr  \Biggr|\Biggr|^2 dx \\ 
    & = 2M_r\int_{-M_r}^{M_r} \Biggl| \Biggl| h{''}(1 - G) \left(\widehat{\nabla} G\right) \left(\widehat{\nabla} G\right)\tr - h{''}(1 - G) \left(\widehat{\nabla} G\right) \left(\nabla F \right)\tr \\
    & \qquad \qquad \qquad\qquad \qquad+ h{''}(1 - G) \left(\widehat{\nabla} G\right) \left(\nabla F \right)\tr -  h{''}(1 - F) \left(\nabla F\right)\left(\nabla F\right)\tr  \Biggr|\Biggr|^2 dx \\
    &  \overset{(c)}{\leq} 4M_r\int_{-M_r}^{M_r} \Biggl| \Biggl| h{''}(1 - G) \left(\widehat{\nabla} G\right) \left(\widehat{\nabla} G\right)\tr - h{''}(1 - G) \left(\widehat{\nabla} G\right) \left(\nabla F \right)\tr\Biggr|\Biggr|^2 dx\\
    & \qquad \qquad \qquad + 4M_r\int_{-M_r}^{M_r} \Biggl|\Biggl|h{''}(1 - G) \left(\widehat{\nabla} G\right) \left(\nabla F \right)\tr -  h{''}(1 - F) \left(\nabla F\right)\left(\nabla F\right)\tr  \Biggr|\Biggr|^2 dx \\
    & \overset{(d)}{\leq} 4M_r\int_{-M_r}^{M_r} \Biggl|\Biggl| h{''}(1 - G) \left(\widehat{\nabla} G\right) \Biggr|\Biggr|^2 \Biggl|\Biggl| \left(\widehat{\nabla} G\right)\tr - \left(\nabla F \right)\tr \Biggr|\Biggr|^2 dx \\
    & \qquad \qquad \qquad \qquad \qquad + 4M_r\int_{-M_r}^{M_r} \Biggl|\Biggl|h{''}(1 - G) \left(\widehat{\nabla} G\right) -  h{''}(1 - F) \left(\nabla F\right)  \Biggr|\Biggr|^2 \Biggl| \Biggl| \left(\nabla F\right)\tr\Biggr|\Biggr|^2 dx \\
    & \overset{(e)}{\leq} 4M_r\int_{-M_r}^{M_r} M_{h''}^2  T^2 M_d^2 \Biggl|\Biggl| \left(\widehat{\nabla} G\right) - \left(\nabla F \right) \Biggr|\Biggr|^2 dx\\
    & \qquad \qquad \qquad \qquad \qquad+ 4M_r\int_{-M_r}^{M_r} T^2 M_d^2 \Biggl|\Biggl|h{''}(1 - G) \left(\widehat{\nabla} G\right) -  h{''}(1 - F) \left(\nabla F\right)  \Biggr|\Biggr|^2 dx, \numberthis \label{eq:d1_1}
\end{align*}
Where $(b)$ follows from the Cauchy-Schwarz inequality for integrals, $(c)$ follows from the fact that $\|X + Y\|^2 \leq 2(\|X\|^2 + \|Y\|^2)$, $(d)$ follows from Cauchy-Schwarz inequality, and  $(e)$ follows \Cref{ass:g_bound} and \Cref{lem:cdfbounds}. 

Notice that  
\begin{align*}
&
\Biggl|\Biggl|h{''}(1 - G) \left(\widehat{\nabla} G\right) -  h{''}(1 - F) \left(\nabla F\right)  \Biggr|\Biggr|^2 \\
& = \Biggl|\Biggl|h{''}(1 - G) \left(\widehat{\nabla} G\right) - h{''}(1 - F) \left(\widehat{\nabla} G\right) + h{''}(1 - F) \left(\widehat{\nabla} G\right) -  h{''}(1 - F) \left(\nabla F\right)  \Biggr|\Biggr|^2 \\
& \overset{(f)}{\leq} 2 \Biggl|h{''}(1 - G) - h{''}(1 - F) \Biggr|^2\Biggl|\Bigg|\left(\widehat{\nabla} G\right)\Biggr|\Biggr|^2 + 2\Biggl|h{''}(1 - F)\Biggr|^2 \Biggl|\Biggl|\left(\widehat{\nabla} G\right) -  \left(\nabla F\right)  \Biggr|\Biggr|^2 \\
&\overset{(g)}{\leq} 2 M_{h^{'''}}^2 \Biggl| G - F \Biggr|^2 \Biggl|\Bigg|\left(\widehat{\nabla} G\right)\Biggr|\Biggr|^2 + 2 M_{h^{''}}^2 \Biggl|\Biggl|\left(\widehat{\nabla} G\right) -  \left(\nabla F\right)  \Biggr|\Biggr|^2 \\
& \overset{(h)}{\leq} 2 M_{h^{'''}}^2 T^2 M_d^2 \Biggl| G - F \Biggr|^2  + 2 M_{h^{''}}^2 \Biggl|\Biggl|\left(\widehat{\nabla} G\right) -  \left(\nabla F\right)  \Biggr|\Biggr|^2, \numberthis \label{eq:d1_2}
\end{align*}
where $(f)$ follows from the property of norm and $\|X + Y\|^2 \leq 2(\|X\|^2 + \|Y\|^2)$, $(g)$ follows from the mean-value theorem and \Cref{ass:g_bound}, and $(h)$ follows from \Cref{lem:cdfbounds}. 

We now bound the term $\Biggl|\Bigg|\int_{-M_r}^{M_r} D_2 dx \Biggr|\Bigg|^2$ in \eqref{eq:D1D2} as follows: 
\begin{align*}
    \Biggl|\Bigg|\int_{-M_r}^{M_r} D_2 dx \Biggr|\Bigg|^2& = \Biggl|\Biggl|\int_{-M_r}^{M_r} h'(1 - F)\nabla^2 F - g'(1 - G) \widehat{\nabla}^2 G dx \Biggr|\Bigg|^2 \\
    & = \Biggl|\Biggl|\int_{-M_r}^{M_r} h'(1 - F)\nabla^2 F - h'(1 - G)\nabla^2 F + h'(1 - G)\nabla^2 F - h'(1 - G) \widehat{\nabla}^2 G dx \Biggr|\Biggr|^2 \\
    & \overset{(i)}{\leq} 2M_r\int_{-M_r}^{M_r} \Biggl|\Bigg| h'(1 - F)\nabla^2 F - h'(1 - G)\nabla^2 F + h'(1 - G)\nabla^2 F - h'(1 - G) \widehat{\nabla}^2 G  \Biggr|\Biggr|^2 dx \\
    & \overset{(j)}{\leq} 4M_r \int_{-M_r}^{M_r} \Biggl[ \Biggl|h'(1 - F) - h'(1 - G)\Biggr|^2 \Biggl|\Biggl| \nabla^2 F\Biggr|\Biggr|^2 +  \Biggl|h'(1 - G) \Biggr|^2 \Biggl|\Biggl| \nabla^2 F - \widehat{\nabla}^2 G \Biggr|\Biggr|^2\Biggr] dx \\
    & \overset{(k)}{\leq} 4M_r \int_{-M_r}^{M_r} \Biggl[ M_{h^{''}}^2 (T M_h + T^2 M_d^2)^2 \Biggl|F -G\Biggr|^2  +  M_{h'}^2 \Biggl|\Biggl| \nabla^2 F - \widehat{\nabla}^2 G \Biggr|\Biggr|^2\Biggr] dx, \numberthis \label{eq:d2_1}
\end{align*}
where $(i)$ follows from the Cauchy-Schwarz inequality for integrals, $(j)$ follows from the fact that $\|X + Y\|^2 \leq 2(\|X\|^2 + \|Y\|^2)$, and $(k)$ follows from mean-value theorem and \Cref{lem:cdfbounds}. 

Substituting \eqref{eq:d1_1}, \eqref{eq:d1_2} and \eqref{eq:d2_1} in \eqref{eq:D1D2}, we obtain
\begin{align*}
    &
    \E \left[\left\| \widehat{\nabla}^2 \rho_h(\theta) - \nabla^2 \rho_h(\theta)\right\|^2\right] \\
    & \overset{(l)}{\leq} 4M_r\E \Biggl[ \int_{-M_r}^{M_r} \Biggl\{ 2M_{h''}^2  T^2 M_d^2 \Biggl|\Biggl| \left(\widehat{\nabla} G\right)\tr - \left(\nabla F \right)\tr \Biggr|\Biggr|^2 + 4 T^4 M_d^4 M_{h^{'''}}^2 \Biggl| G - F \Biggr|^2  \\
    & + 4 T^2 M_d^2 M_{h^{''}}^2 \Biggl|\Biggl|\left(\widehat{\nabla} G\right) -  \left(\nabla F\right)  \Biggr|\Biggr|^2 + 2 \Biggl( M_{h^{''}}^2 (T M_h + T^2 M_d^2)^2 \Biggl|F -G\Biggr|^2  +  M_{h'}^2 \Biggl|\Biggl| \nabla^2 F - \widehat{\nabla}^2 G \Biggr|\Biggr|^2\Biggr) \Biggr\} dx \Biggr] \\
    & \overset{(m)}{\leq} 4M_r\int_{-M_r}^{M_r} \Biggl\{ 2M_{h''}^2  T^2 M_d^2 \E \Biggl|\Biggl| \left(\widehat{\nabla} G\right)\tr - \left(\nabla F \right)\tr \Biggr|\Biggr|^2 + 4 T^4 M_d^4 M_{h^{'''}}^2 \E \Biggl| G - F \Biggr|^2  \\
    & + 4T^2 M_d^2 M_{h^{''}}^2  \E \Biggl|\Biggl|\left(\widehat{\nabla} G\right) -  \left(\nabla F\right)  \Biggr|\Biggr|^2 +  2\Biggl( M_{h^{''}}^2 (T M_h + T^2 M_d^2)^2 \E \Biggl|F - G\Biggr|^2  +  M_{h'}^2 \E \Biggl|\Biggl| \nabla^2 F - \widehat{\nabla}^2 G \Biggr|\Biggr|^2\Biggr) \Biggr\} dx \\
    & \overset{(n)}{\leq} 4M_r\int_{-M_r}^{M_r} \Biggl\{ 2M_{h''}^2  T^2 M_d^2 \left(\frac{4e^2T^2 M^2_d}{m}\right) + 4 T^4 M_d^4 M_{h^{'''}}^2 \left(\frac{4}{m}\right) \\
    &\qquad + 2 T^2 M_d^2 M_{h^{''}}^2 \left(\frac{4e^2T^2 M^2_d}{m}\right) + 2 \Biggl( M_{h^{''}}^2 (T M_h + T^2 M_d^2)^2 \left(\frac{4}{m}\right) +  M_{h'}^2 \E \Biggl|\Biggl| \nabla^2 F - \widehat{\nabla}^2 G \Biggr|\Biggr|^2\Biggr) \Biggr\} dx \\
    & = \frac{\kappa_2}{m} +8M_r M^2_{h'}\int_{-M_r}^{M_r}\E \Biggl|\Biggl| \nabla^2 F - \widehat{\nabla}^2 G \Biggr|\Biggr|^2 dx, \numberthis \label{eq:err_tem}
\end{align*}
where $\kappa_2 = 64M_r^2\left(3e^2 M^2_{h''}T^4M_d^4 + 2 T^4 M_d^4 M^2_{h'''} +  M^2_{h''}(TM_h + T^2M^2_d)^2\right).$ In the above, $(l)$ follows from the \eqref{eq:d1_1}, \eqref{eq:d1_2} and \eqref{eq:d2_1}, $(m)$ follows from Fubini's theorem, and $(n)$ follows from \Cref{propo_grad2}. 

We now bound the term $\int_{-M_r}^{M_r}\E  \|\nabla^2 F - \widehat{\nabla}^2 G \|^2 dx$ in \eqref{eq:err_tem}.
From \eqref{eq:edf_Hessian}, we have 
\begin{align*}
    \widehat{\nabla}^2 G^{m}_{R^{\theta}}(x)
    & = \frac{1}{b} \sum_{i =1}^{b} \Biggl[ \mathbf{1}\{R_i^{\theta} \leq x\}\Bigg\{\sum_{t = 0}^{T - 1} \nabla^2 \log \pi_{\theta}(A^i_t| S^i_t) + \left(\sum_{t = 0}^{T - 1} \nabla \log \pi_{\theta}(A^i_t| S^i_t)\right) \left(\sum_{t = 0}^{T - 1} \nabla \log \pi_{\theta}(A^i_t| S^i_t)\right)\tr \Biggr\} \Biggr]. \\
    & = \frac{1}{b} \sum_{i =1}^{b} \mathcal{H}_i(\theta). 
\end{align*}
From \Cref{lem:cdfbounds}, we know that $\E \left[\|\mathcal{H}_i(\theta)\right\|^2] \leq (TM_h + T^2 M_d^2)^2$. 

From Theorem 1 in \cite{tropp2016expected}, we have
\begin{equation}
    \E \Biggl|\Biggl| \nabla^2 F - \widehat{\nabla}^2 G \Biggr|\Biggr|^2 \leq \frac{2C(d)}{b^2}\left(\left \|\sum \E[\Delta_i^2] \right \| + C(d) \E \left[\max \left \|\Delta_i \right \|^2\right]\right),
\end{equation}
where $\Delta_i = \mathcal{H}_i(\theta) - \nabla^2 F$ and $C(d) = 4(1+2\log 2d)$. 

It is apparent that $\E [\|\Delta_i\|^2] \leq \E \left[\|\mathcal{H}_i(\theta)\right\|^2] \leq (TM_h + T^2 M_d^2)^2 $ and $\|\sum \E[\Delta_i^2]\| \leq \sum \|\E[\Delta_i^2]\| \leq \sum \E [\|\Delta_i^2\|]$. 
Thus,
\begin{equation}
\label{eq:tropp_1}
    \E \Biggl|\Biggl| \nabla^2 F - \widehat{\nabla}^2 G \Biggr|\Biggr|^2 \leq \frac{4C(d)}{b} (TM_h + T^2 M_d^2)^2. 
\end{equation}
Using Rosenthal's inequality and Lemma 15 from \cite{maniyar2024crpn}, we get
\begin{equation}
\label{eq:tropp_2}
    \E \Biggl|\Biggl| \nabla^2 F - \widehat{\nabla}^2 G \Biggr|\Biggr|^4 \leq \frac{15d^2}{b^2} (TM_h + T^2 M_d^2)^4. 
\end{equation}
Substituting \eqref{eq:tropp_1} in \eqref{eq:err_tem}, we get
\begin{equation}
\label{eq:err_2}
    \E \left[\left\| \widehat{\nabla}^2 \rho_h(\theta) - \nabla^2 \rho_h(\theta)\right\|^2\right] \leq \left(\frac{\kappa_2}{m} + \frac{64M_r^2M^2_{h'}C(d)(TM_h + T^2 M_d^2)^2}{b}\right). 
\end{equation}
 Similarly using \eqref{eq:tropp_2}, \Crefrange{propo_grad4}{propo_Hessian4} and the inequality $\|X+Y\|_F^4 \leq 8\left(\|X\|_F^4 + \|Y\|_F^4\right)$,  
we have

\begin{equation}
\label{eq:err_4}
    \E \left[\left\| \widehat{\nabla}^2 \rho_h(\theta) - \nabla^2 \rho_h(\theta)\right\|_F^4\right] \leq \left(\frac{\kappa_3}{m^2} + \frac{1920M_r^2M^4_{h'}d^2(TM_h + T^2 M_d^2)^4}{b^2}\right), 
\end{equation}
where $\kappa_3 = 4096M_r^2\left\{T^8M_d^8(9e^2M_{h''}^4 + 8 M_{h'''}^4)+M_{h'}^4(TM_d+T^4M_d^4)^4\right\}$  and $\|.\|_{F}$ denotes the Frobenius norm. 

Thus, by using H\"{o}lder's inequality, we obtain
\begin{align*}
&
    \E \left[\left\| \widehat{\nabla}^2 \rho_h(\theta) - \nabla^2 \rho_h(\theta)\right\|^3\right]\\
    &\leq \left(\E \left[\left\| \widehat{\nabla}^2 \rho_h(\theta) - \nabla^2 \rho_h(\theta)\right\|_F^4\right]. \E \left[\left\| \widehat{\nabla}^2 \rho_h(\theta) - \nabla^2 \rho_h(\theta)\right\|^2\right]\right)^{\frac{1}{2}}\\
    & \overset{(o)}{\leq} \left(\frac{\kappa_3}{m^2} + \frac{1920M_r^4M_{h''}^4d^2(TM_h + T^2 M_d^2)^4}{b^2}\right)^{\frac{1}{2}}.\left(\frac{\kappa_2}{m} + \frac{64M_r^2M^2_{h'}C(d)(TM_h + T^2 M_d^2)^2}{b}\right)^{\frac{1}{2}}, 
\end{align*} 
where $(o)$ follows from \eqref{eq:err_2} and \eqref{eq:err_4}. 

If $m \geq b$, then 
\begin{align*}
&
    \E \left[\left\| \widehat{\nabla}^2 \rho_h(\theta) - \nabla^2 \rho_h(\theta)\right\|^3\right] \leq \left(\frac{\kappa_3 + t_2}{b^2}\right)^{\frac{1}{2}}.\left(\frac{\kappa_2 + t_1}{b}\right)^{\frac{1}{2}} = \frac{\sqrt{(\kappa_2 + t_1)(\kappa_3 + t_2)}}{b^{\frac{3}{2}}}, 
\end{align*}
where $t_1 = 64M_r^2M^2_{h'}C(d)(TM_h + T^2 M_d^2)^2 $ and $t_2 = 1920M_r^2M_{h'}^4d^2(TM_h + T^2 M_d^2)^4.$ 
The claim follows.
\end{proof}

\begin{lemma}
\label{lemma:lemma9}
Let the parameters $\{\theta_k\}$ be generated by \Cref{alg:CRPN_drm}. Then, we have 
\begin{align*}
&
    \sqrt{\E\left[\left \|\theta_{k} - \theta_{k-1}\right \|^2\right]}\\
    & \geq \max \Biggl\{\sqrt{\frac{\E\left[\|\nabla \rho_h(\theta_k)\|\right] - \delta_k^g - \delta_k^{\mathcal{H}}}{L_{\mathcal{H}} + \alpha_k}}, \frac{2}{\alpha_k + 2 L_{\mathcal{H}}}\Biggl[\E \left[\lambda_{\max}(\nabla^2 \rho_h(\theta_k))\right]- \sqrt{2(\alpha + L_{\mathcal{H}}) \delta_k^{\mathcal{H}}}\Biggr]\Biggr\},  
\end{align*}
where $\delta_k^g > 0$ and $\delta_k^{\mathcal{H}}> 0$ are defined as follows:
\begin{eqnarray}
\label{eq:error_bound}
    \E\left[\left \|\nabla \rho_h(\theta_{k-1}) - \widehat{\nabla}_k \rho_{h}\right \|^2\right] \leq (\delta^g_k)^2, \; \; \; \E\left[\left \|\nabla^2 \rho_h\theta_{k-1}) - \widehat{\nabla}^2_k \rho_h\right \|^3  \right] \leq \Bigl(2(L_{\mathcal{H}} + \alpha_k)\delta_k^{\mathcal{H}}\Bigr)^{\frac{3}{2}}. 
\end{eqnarray}
\end{lemma}
\begin{proof}
    This lemma can be proved by  completely parallel arguments to the proof of Lemma $9$ in \cite{maniyar2024crpn}. 
\end{proof} 

\begin{lemma}
\label{lemma:lemma10}
Let the parameters $\{\theta_k\}$ be generated by the \Cref{alg:CRPN_drm}. Then, for a given iteration limit $N\geq 1$, we have
\begin{align*}
&
    \E \left[\left\|\theta_{R} - \theta_{R-1}\right\|^3\right] \leq \frac{36}{\sum_{k = 1}^{N} \alpha_k} \left[\rho_{h}^{*}-\rho_h(\theta_0) + \sum_{k=1}^{N} \frac{4(\delta_k^g)^{\frac{3}{2}}}{\sqrt{3\alpha_k}} + \sum_{k=1}^N \left(\frac{18 \sqrt[4]{2}}{\alpha_k}\right)^2 ((L_{\mathcal{H}} + \alpha_k)\delta_k^{\mathcal{H}})^{\frac{3}{2}}\right],     
\end{align*}
where $R$ is a integer random variable supported on $\{1,\ldots,N\}$, whose probability distribution is defined as follows:
\[P_R(R = k) = \frac{\alpha_k}{\sum_{k=1}^{N}\alpha_k}, \;\;\; k = 1\ldots, N,\] and $\delta_k^g, \delta_k^{\mathcal{H}}$ are satisfying the conditions as before in \eqref{eq:error_bound}.  

\end{lemma} 

\begin{proof}
    See Lemma 10 in \cite{maniyar2024crpn}. 
\end{proof}

\subsection{Proof of Theorem \ref{thm: convergence_SOSP}}
\label{proof: convergence_SOSP}
\begin{proof}
We follow the technique from \cite{maniyar2024crpn} to prove this theorem.  The proof proceeds through the following sequence of lemmas: (i) First, we use \Cref{lemma:error_bound_gradient} and \Cref{lemma: error_bound_Hessian} with the parameters chosen as specified in \eqref{eq:hyper_parameter}, (ii) Note that \eqref{eq:error_bound} in \Cref{lemma:lemma9} is satisfied with $\delta_k^g = \frac{2\epsilon}{5}$ and $\delta_k^{\mathcal{H}} = \frac{\epsilon}{144}$, (iii) We apply \Cref{lemma:lemma10} with the chosen parameters as specified in \eqref{eq:hyper_parameter}; and (iv) We apply \Cref{lemma:lemma9} with $\delta_k^g$ and $\delta_k^{\mathcal{H}}$ set as mentioned above. 


From \Cref{lemma:lemma10} and using \eqref{eq:hyper_parameter}, we have
\begin{align}
    \E \Bigl[\left\|\theta_R - \theta_{R-1}\right\|^3\Bigr] \leq \frac{8\epsilon^{3/2}}{L_{{\mathcal{H}}}^{3/2}}. 
\end{align}
Using Lyapunov inequality, we obtain
\begin{align}
    \E \Bigl[\left\|\theta_R - \theta_{R-1}\right\|^2\Bigr]^{1/2} \leq \E \Bigl[\left\|\theta_R - \theta_{R-1}\right\|^3\Bigr]^{1/3} \leq \frac{2\epsilon^{1/2}}{L_{{\mathcal{H}}}^{1/2}}. 
\end{align}

The main proof follows from \Cref{lemma:lemma9} and 
\[\sqrt{\E[\|\nabla \rho_h(\theta_k)\|}] \leq 5 \sqrt{\epsilon},\;\;\; \frac{\E [\lambda_{\max}(\nabla^2 \rho_h(\theta_k))]}{\sqrt{L_{\mathcal{H}}}} \leq 6 \sqrt{\epsilon}.\] 

The sample complexity of CRPN-DRM is bounded by
\[\sum_{k = 1}^N m_k = \mathcal{O}\Bigl(\frac{1}{\epsilon^{3.5}}\Bigr), \;\;\; \sum_{k = 1}^N b_k = \mathcal{O}\Bigl(\frac{1}{\epsilon^{2.5}}\Bigr).\]

\end{proof}

\section{Efficient Hessian-vector product computation using Auto-grad in the DRM-sensitive RL setting}
\label{sec:eff_hessian}
We follow the approach adopted in \cite{maniyar2024crpn} for the efficient implementation of \Cref{alg:CRPN_drm}. 
For the sake of completeness, we provide the efficiency calculations below.


Our variance-reduced DRM policy Hessian estimate consists of three quantities which we shall discuss below. 
From \Cref{lemma: drm_sample_estimator_alt}, let $\mathbf{L_1}^{(i)} := R^{\theta}_{(i)} h'(1-\frac{i}{m}) l^{\theta}_{(i)}$, $\mathbf{L_2}^{(i)} := \mathbf{S_2}^{(i)} := l^{\theta}_{(i)}$ and $\mathbf{S_1}^{(i)} = \psi{''}_{i} l^{\theta}_{(i)}$ denote the stacked ``losses'' corresponding to each of the $n$ trajectories in the mini-batch.
Given $\theta \in \mathbb{R}^d$ and $\nabla \equiv \nabla_{\theta}$, let $\nabla \mathbf{L_1}^{(ij)} := \frac{\partial \mathbf{L_1}^{(j)}}{\partial \theta^{(i)}}$, $\nabla \mathbf{L_2}^{(ij)} := \frac{\partial \mathbf{L_2}^{(j)}}{\partial \theta^{(i)}}$, $\nabla \mathbf{S_1}^{(ij)} := \frac{\partial \mathbf{S_1}^{(j)}}{\partial \theta^{(i)}}$ and $\nabla \mathbf{S_2}^{(ij)} := \frac{\partial \mathbf{S_2}^{(j)}}{\partial \theta^{(i)}}$.
Then, $\mathbf{L_1}, \mathbf{L_2}, \mathbf{S_1}, \mathbf{S_2} \in \mathbb{R}^n$ and $\nabla \mathbf{L_1}, \nabla \mathbf{L_2}, \nabla \mathbf{S_1}, \nabla \mathbf{S_2} \in \mathbb{R}^{d \times n}$ by construction. We can now re-write our DRM estimates in the following matrix-vector form
\begin{align*}
	\Bar{g} &:= \widehat{\nabla} \rho_h = \frac{1}{n} \sum_{i} \nabla \mathbf{L_1} ^{(i)} =  \nabla \left( \frac{1}{n} \sum_{i}  \mathbf{L_1}^{(i)} \right),\\
	\Bar{\Hess} &:= \widehat{\nabla}^2 \rho_h =   \frac{1}{n} \sum_{i} (\nabla \mathbf{S_1}^{(i)} \nabla^{\top} \mathbf{S_2}^{(i)} + \nabla \mathbf{L_1}^{(i)} \nabla^{\top} \mathbf{L_2}^{(i)} + \nabla^2 \mathbf{L_1}^{(i)} ) \\
	&= \frac{1}{n} \sum_{i} (\nabla \mathbf{S_1}^{(i)} \nabla^{\top} \mathbf{S_2}^{(i)}) + \frac{1}{n} \sum_{i} (\nabla \mathbf{L_1}^{(i)} \nabla^{\top} \mathbf{L_2}^{(i)}) + \frac{1}{n} \sum_{i} \nabla^2 \mathbf{L_1}^{(i)}.
\end{align*}
The resulting Hessian-vector product can be shown as the following in matrix form,
\begin{align*}
	\Bar{\Hess} \cdot v &= \frac{1}{n} \sum_{i} (\nabla \mathbf{S_1}^{(i)} \nabla^{\top} \mathbf{S_2}^{(i)}) \cdot v + \frac{1}{n} \sum_{i} (\nabla \mathbf{L_1}^{(i)} \nabla^{\top} \mathbf{L_2}^{(i)}) \cdot v + \frac{1}{n} \sum_{i}\nabla^2 \mathbf{L_1}^{(i)} \cdot v ,\\
	&= \frac{1}{n} \sum_{i} \nabla \mathbf{S_1}^{(i)} (\nabla^{\top} \mathbf{S_2}^{(i)} v) + \frac{1}{n} \sum_{i} \nabla \mathbf{L_1}^{(i)} (\nabla^{\top} \mathbf{L_2}^{(i)} v) + \nabla \frac{1}{n} \sum_{i} (\nabla^{\top} \mathbf{L_1}^{(i)} v) , \\
	&= \frac{1}{n} \nabla \mathbf{S_1} \nabla^{\top} \mathbf{S_2} v + \frac{1}{n} \nabla \mathbf{L_1} \nabla^{\top} \mathbf{L_2} v + \nabla \Bar{g}^{\top} v . \numberthis\label{eq:hvp}
\end{align*}
The last term, $\nabla \Bar{g}^{\top} v$ can be computed by utilizing the computational graph of $\Bar{g}$ obtained during gradient calculation thereby requiring just \textbf{one} additional auto-grad call. However, the first two terms, given their similar forms, can be calculated by using the forward auto-diff trick shown in Appendix I of \cite{maniyar2024crpn}.

\section{Additional simulation experiments}
\label{sec:appendix-expts}
We present additional simulation results on the three Gymnasium environments in the following sections, while a summary of the hyperparameters used in our experiments is provided in \Cref{tab:hyperparameter_summary}. 

\begin{table}[h]
    \caption{Summary of all the hyperparameters used for each of the environments.}
    \label{tab:hyperparameter_summary}
    \centering
    \begin{tabular}{|c|c|c|c|}
    \hline
    \multirow{2}{*}{}  & \multirow{2}{*}{CartPole-v1} & \multirow{2}{*}{Humanoid-v4} & \multirow{2}{*}{CliffWalking-v0} \\
     & & & \\ \hline
    \multirow{2}{*}{Observation space $\mathbb{S}$} & \multirow{2}{*}{Continuous $\mathbb{R}^4$} & \multirow{2}{*}{Continuous $\mathbb{R}^{348}$} & \multirow{2}{*}{Discrete $|\mathbb{S}| = 48$}\\
     & & & \\ \hline
     \multirow{2}{*}{Action space $\mathbb{A}$} &  \multirow{2}{*}{Discrete $|\mathbb{A}| = 4$} &  \multirow{2}{*}{Continuous $\mathbb{R}^{17}$} &  \multirow{2}{*}{Discrete $|\mathbb{A}| = 4$} \\
      & & & \\ \hline
     \multirow{2}{*}{Policy $\pi(s|a ; \theta)$}   &  \multirow{2}{*}{Linear Boltzmann} &  \multirow{2}{*}{Deep Gaussian} &  \multirow{2}{*}{Tabular Boltzmann} \\
      & & & \\ \hline
    \multirow{2}{*}{Total iterations $N$} & \multirow{2}{*}{100} & \multirow{2}{*}{100} & \multirow{2}{*}{1000} \\
      & & & \\ \hline
    \multirow{2}{*}{Gradient batch size $m_k$} & \multirow{2}{*}{200} & \multirow{2}{*}{200} & \multirow{2}{*}{200} \\
      & & & \\ \hline
    \multirow{2}{*}{Hessian batch size $b_k$} & \multirow{2}{*}{200} & \multirow{2}{*}{200} & \multirow{2}{*}{200} \\
      & & & \\ \hline
    \multirow{2}{*}{Cubic penalty $\alpha$} & \multirow{2}{*}{5000} & \multirow{2}{*}{10$^5$} & \multirow{2}{*}{2500} \\
      & & & \\ \hline
    \multirow{2}{*}{Discount factor $\gamma$} & \multirow{2}{*}{0.99} & \multirow{2}{*}{0.99} & \multirow{2}{*}{1} \\
      & & & \\ \hline
    \end{tabular}
\end{table}
\subsection{Cart-pole and Humanoid}
We implemented \Cref{alg:CRPN_drm} on the Cart-pole Gymnasium environment. From the results in \Cref{fig:cartpole-results}, we observe a trend that is similar to the Humanoid environment, whose results were presented in \Cref{fig:humanoid-results}. In particular, as in the Humanoid case, the DRMACRPN algorithm outperformed the risk-neutral counterpart ACRPN for each of the distortion functions considered.

\begin{figure}[ht]
\begin{tabular}{cc}
\begin{subfigure}{0.5\textwidth} 
    \includegraphics[width=0.9\linewidth]{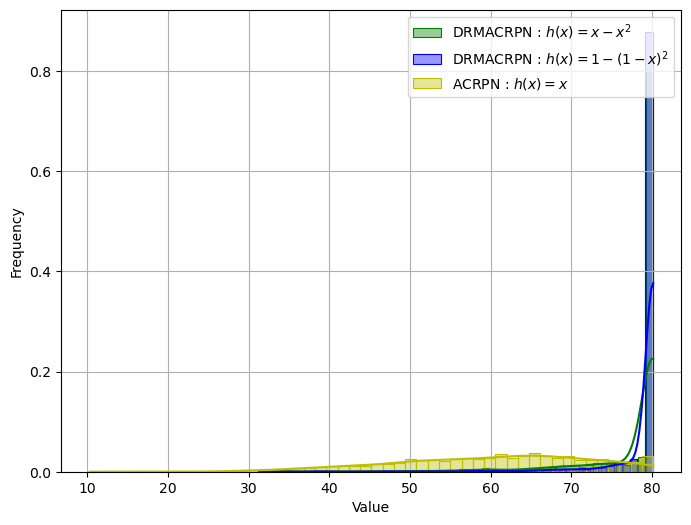}
    \caption{Cumulative return distributions of the DRM policies}
    \label{fig:drm_hists_cartpole}
\end{subfigure}
&
\begin{subfigure}{0.5\textwidth} 
\includegraphics[width=0.9\linewidth]{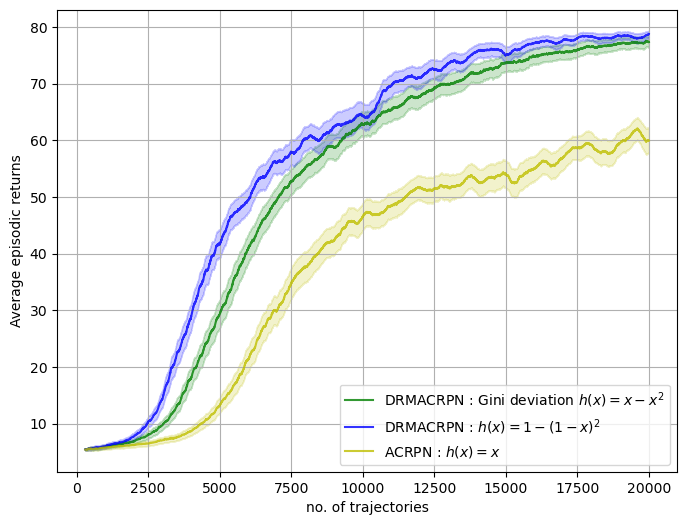}
\caption{Learning curves of the DRM policies.}
    \label{fig:drm_lc_cartpole}
\end{subfigure}
\end{tabular}
\caption{Performance comparison of risk-neutral policy Newton and \Cref{alg:CRPN_drm} with two different DRMs, namely dual-power, Gini deviation and RDEU on Cart-pole.}
\label{fig:cartpole-results}
\end{figure}
\begin{figure}[ht]
\begin{tabular}{cc}
\begin{subfigure}{0.5\textwidth} 
    \includegraphics[width=0.85\linewidth]{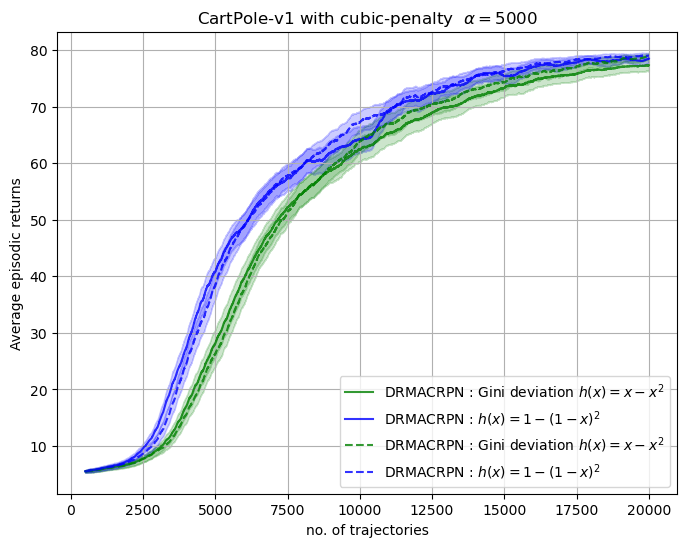}
    \caption{}
    \label{fig:drm_vs_bs2_cartpole}
\end{subfigure}
&
\begin{subfigure}{0.5\textwidth} 
\includegraphics[width=0.85\linewidth]{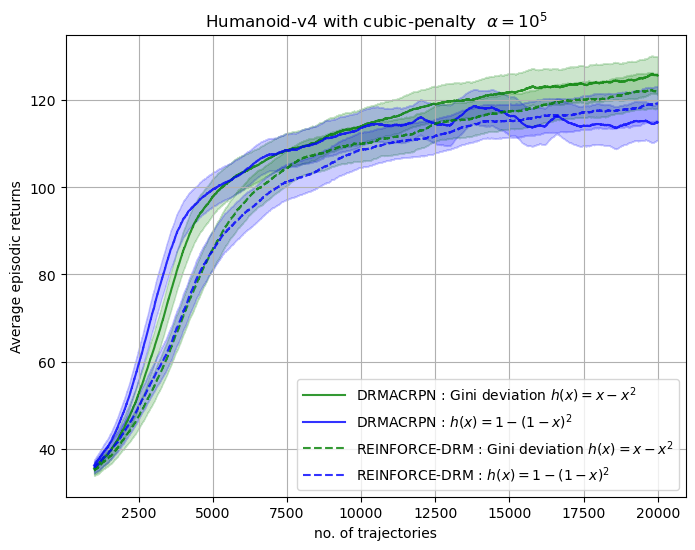}
\caption{}
    \label{fig:drm_vs_bs2_humanoid}
\end{subfigure}
\end{tabular}
\caption{Comparison of learning curves of REINFORCE-DRM and DRMACRPN algorithms on Cart-pole and Humanoid environments.}
\label{fig:drm_vs_firstorder-results}
\end{figure}

We also tested the performance of DRMACRPN and REINFORCE-DRM  with two DRM functions, namely Gini-deviation and Dual-power, on Cart-pole and Humanoid environments. For this comparison, on the Cart-pole environment, we used a linear layer with $\alpha=5000$ and for the Humanoid environment, a $64 \times 64$ deep Gaussian MLP policy with $\alpha=10^5$. Both environments had $m_k = b_k = 200$ with $N = 100$ and were run for ten independent iterations with a discount factor $\gamma = 0.99$. \Cref{fig:drm_vs_firstorder-results} presents the performance of these algorithms. On the Cart-pole environment, we do not observe a significant difference in performance as the policy network is simple owing to its linear nature, which probably implies lower occurrences of saddle-points. In the  Humanoid case, DRMACRPN outperforms REINFORCE-DRM in the initial phase of learning due to the non-linear nature of the deep Gaussian policy network implemented for this environment. We hypothesize the significance of second-order methods to be more helpful for larger policy networks such as deeper networks, transformers, LSTMs, etc, and such extensions can be an interesting direction for future empirical research.


\begin{figure}[ht]
    \centering
    \includegraphics[width=0.5\linewidth]{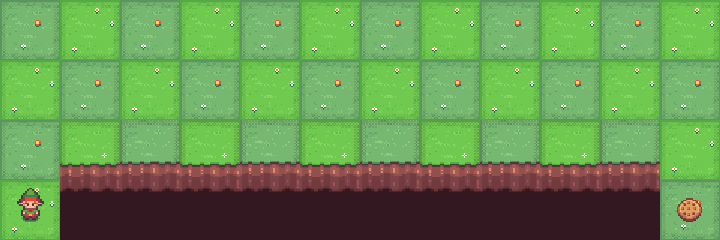}
    \caption{A snapshot of the CliffWalking-v0 environment.}
    \label{fig:cliff_walking}
\end{figure}
\begin{table}[ht]
\caption{DRM values of the learned policies for our DRMACRPN and the ACRPN on Humanoid-v4. }
    \label{tab:drm_results_2}
    \centering
    \begin{tabular}{|c|c|c|}
    \hline
        \multirow{2}{*}{Value} & \multirow{2}{*}{DRMACRPN} & \multirow{2}{*}{ACRPN} \\ 
        & & \\\hline
       Dual-power & \textbf{126.0} &  114.3 \\ \hline
       Gini deviation & \textbf{13.5} & 11.3 \\ \hline
    \end{tabular}
\end{table} 

\subsection{Cliff Walking}
This\footnote{\url{https://gymnasium.farama.org/environments/toy_text/cliff_walking/}} is an environment by OpenAI's gymnasium library that consists of a gridworld with $4 \times 12 = 48$ states and $4$ actions of each state to move up, right, down, and left. The agent starts the episode in the \textit{start} state located on the bottom left of the grid $(3, 0)$ and attempts to reach the \textit{goal} state at the bottom right, i.e. (3, 11). There is a cliff running along the bottom of the gridworld between the \textit{start} and \textit{goal} states. The agent incurs a negative reward of -1 for every time step, unless it falls off the cliff incurring -100, or reaching the goal state where the transition reward is 0. If the agent falls off the cliff, he starts again at the \textit{start} state. The episode ends if and only if the goal state is reached or the time step limit of $250$ is reached. See \Cref{tab:hyperparameter_summary} for a summary for the hyper-parameters used for each of the environments. We run ten independent training loops to obtain ten converged policies for each algorithm.

\begin{figure}[ht]
\begin{tabular}{cc}
\begin{subfigure}{0.5\textwidth} 
    \includegraphics[width=0.9\linewidth]{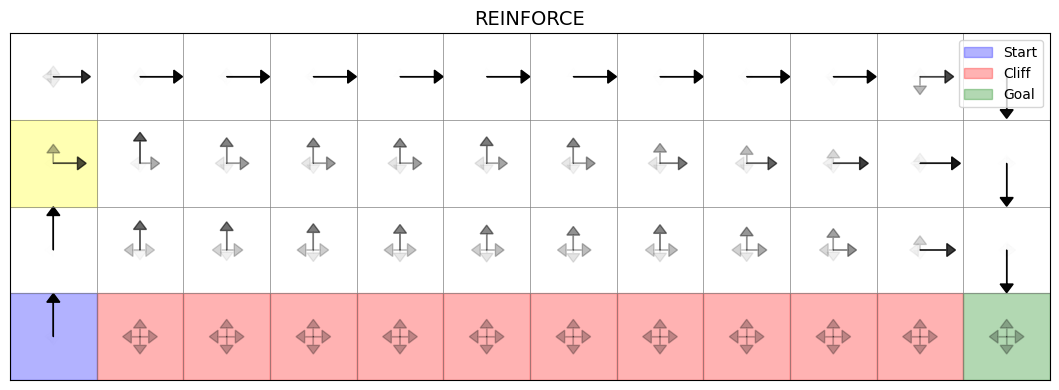}
    \label{fig:cliffwalking_policymap_reinforcedrm_neutral_supp}
\end{subfigure}
&
\begin{subfigure}{0.5\textwidth} 
\includegraphics[width=0.9\linewidth]{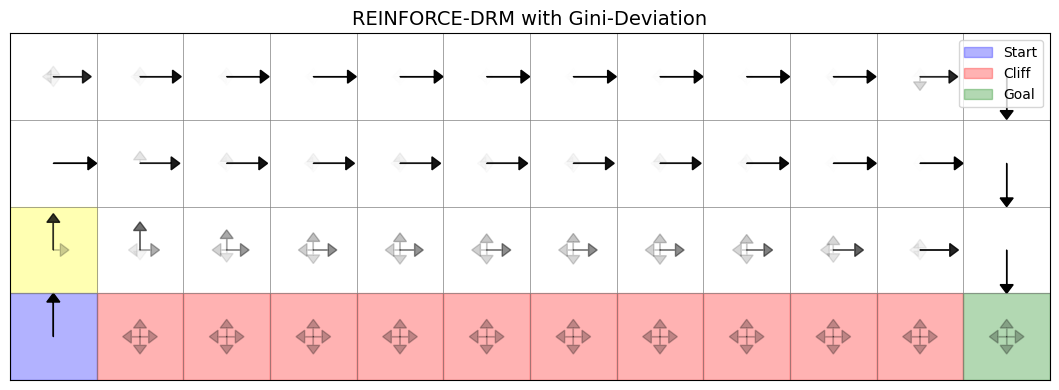}
    \label{fig:cliffwalking_policymap_reinforcedrm_gini_supp}
\end{subfigure}
\\
\begin{subfigure}{0.5\textwidth} 
    \includegraphics[width=0.9\linewidth]{figs/cliffwalking_policymap_drmacrpn_neutral.png}
    \label{fig:cliffwalking_policymap_drmacrpn_neutral_supp}
\end{subfigure}
&
\begin{subfigure}{0.5\textwidth} 
\includegraphics[width=0.9\linewidth]{figs/cliffwalking_policymap_drmacrpn_gini.png}
    \label{fig:cliffwalking_policymap_drmacrpn_gini_supp}
\end{subfigure}
\end{tabular}
\caption{Comparison of the policy maps. The length of the arrows are proportional to the probability of taking that action by the \textit{aggregated} policy $\pi_{agg}$.}
\label{fig:cliffwalking_policymap_results}
\end{figure}


\noindent
It is also interesting to note that risk-seeking policies may tend to fall off the cliff, as is evident by the minimum values of $-123$ from \Cref{tab:cliffwalking_results}. We also visualize the policy maps for REINFORCE and REINFORCE-DRM given in \Cref{fig:cliffwalking_policymap_results}. As we have 10 learned policy parameters $\theta_{i}$, for $i \in [1, 10]$, we obtain the aggregate policy as 
$$\pi_{agg} (s, \cdot) = \frac{1}{10}\sum_{n=1}^{10}\textrm{softmax}(\theta_n(s, \cdot)).$$
We use this $\pi_{agg}$ policy to visualize our policy maps in \Cref{fig:cliffwalking_policymap_results_1} and \Cref{fig:cliffwalking_policymap_results}. In a similar manner to what we did for the Cart-pole and Humanoid environments, we also provide the reward distribution especially for the two risk-seeking policies in \Cref{fig:cliffwalking_barplot} after removing outliers, i.e. cumulative rewards with very low frequencies. From \Cref{fig:cliffwalking_barplot}, we can also conclude that the second-order algorithms are more likely to converge to a more optimal policy. The learning curves for this experiment are shown in \Cref{fig:cliffwalking_lcs} for the sake of completeness. We see that each of the algorithms are assured convergence while also showcasing that the second-order methods achieve faster convergence due to their ability to understand the curvature of the reward landscape better. In conclusion, the performance of risk-seeking DRMs is higher in expectation than their risk-neutral counterparts while second-order tend to outperform first-order by being quicker and likelier to converge to better policies.

\begin{figure}[ht]
\centering
\begin{tabular}{cc}
\begin{subfigure}{0.5\textwidth} 
    \includegraphics[width=0.9\linewidth]{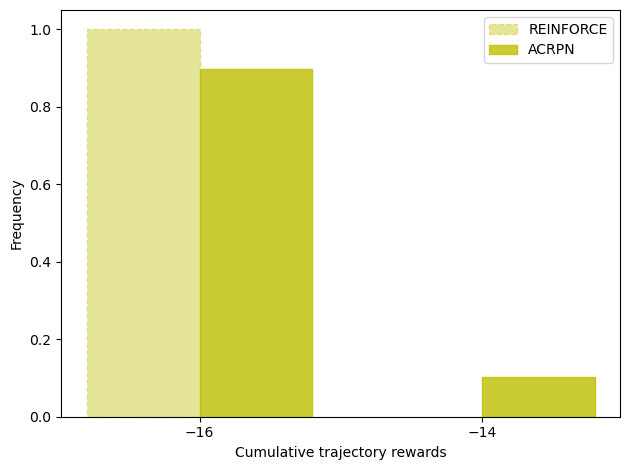}
    \caption{Distributions of the risk-neutral policies.}
    \label{fig:cliffwalking_riskneutral_barplot}
\end{subfigure}
&
\begin{subfigure}{0.5\textwidth}
\includegraphics[width=0.9\linewidth]{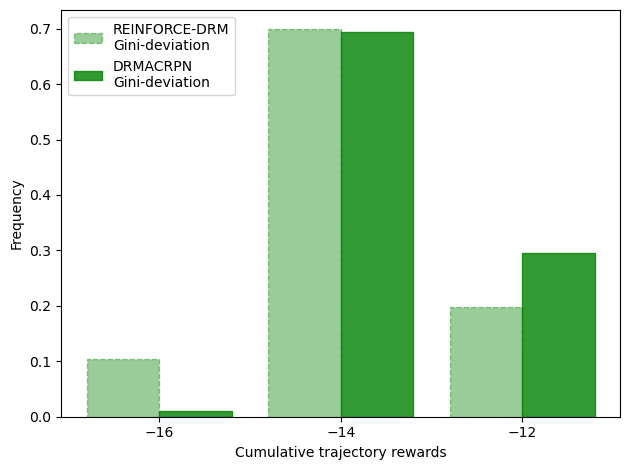}
    \label{fig:cliffwalking_drm_barplot}
    \caption{Distributions of the risk-seeking policies.}
\end{subfigure}
\end{tabular}
\caption{Distributions of the policies comparing first-order vs second-order.}
\label{fig:cliffwalking_barplot}
\end{figure}

\begin{figure}[ht]
\centering
\begin{tabular}{cc}
\begin{subfigure}{0.5\textwidth} 
    \includegraphics[width=0.9\linewidth]{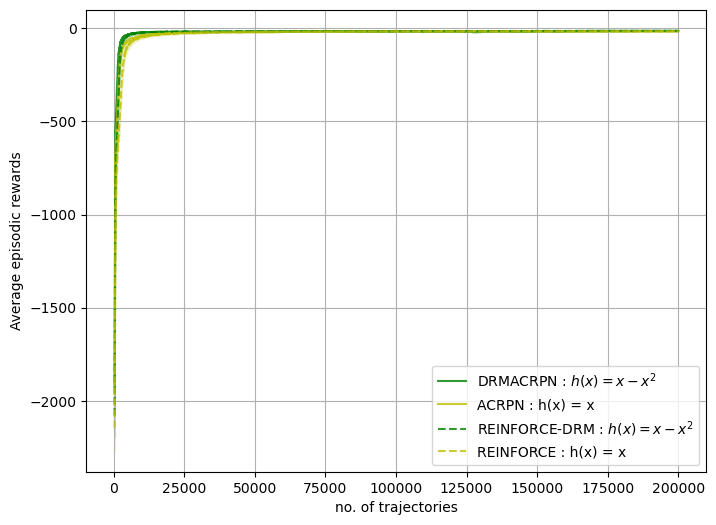}
    \caption{Asymptotic convergence of the learning curves for $N=1000$ iterations.}
    \label{fig:cliffwalking_lc_asymptotic}
\end{subfigure}
&
\begin{subfigure}{0.5\textwidth}
\includegraphics[width=0.9\linewidth]{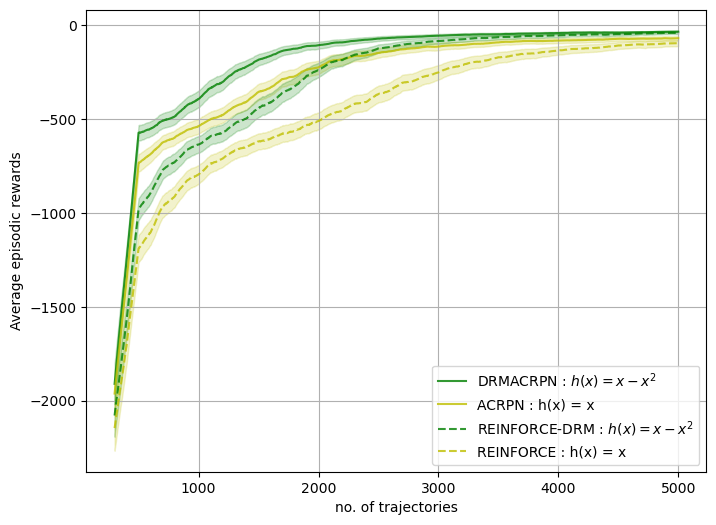}
    \label{fig:cliffwalking_lc_first25}
    \caption{Learning curves for the first $k=[1, 25]$ iterations.}
\end{subfigure}
\end{tabular}
\caption{Plot of the learning curves for the 4 algorithms for CliffWalking environment.}
\label{fig:cliffwalking_lcs}
\end{figure}

\end{document}